\def\year{2022}\relax
\documentclass[letterpaper]{article}
\usepackage{amsmath,amsfonts,stmaryrd,amssymb,enumitem}
\usepackage{booktabs,todonotes}

\usepackage{algpseudocode}
\algnewcommand\algorithmicinput{\textbf{Input:}}
\algnewcommand\INPUT{\item[\algorithmicinput]}

\algnewcommand\algorithmicoutput{\textbf{Output:}}
\algnewcommand\OUTPUT{\item[\algorithmicoutput]}

\usepackage{pifont}

\usepackage{tikz}
\usetikzlibrary{arrows}
\usetikzlibrary{calc}
\usetikzlibrary{shapes}
\usetikzlibrary{fit}
\usetikzlibrary{matrix} 
\usetikzlibrary{positioning}
\usetikzlibrary{decorations.pathreplacing}

\renewcommand{\eqref}[1]{Eq.~(\ref{eq:#1})}
\newcommand{\figref}[1]{Figure~\ref{fig:#1}}

\newcommand{\tabref}[1]{Table~\ref{tab:#1}}        
\newcommand{\secref}[1]{Section~\ref{sec:#1}}
\newcommand{\thmref}[1]{Theorem~\ref{thm:#1}}
\newcommand{\lemref}[1]{Lemma~\ref{lem:#1}}

\newcommand{\appref}[1]{Appendix~\ref{ap:#1}}

\newcommand{\myalgref}[1]{Alg.~\ref{alg:#1}}

\newcommand{\reals}{\mathbb{R}}
\newcommand{\nats}{\mathbb{N}}

\DeclareMathOperator*{\argmax}{argmax}

\newcommand{\cmark}{\ding{51}}%
\newcommand{\xmark}{\ding{55}}%

\newcommand{\cA}{\mathcal{A}}

\newcommand{\cM}{\mathcal{M}}

\usepackage{aaai22arxiv}
\usepackage{times}
\usepackage{helvet}
\usepackage{courier}
\usepackage[hyphens]{url}
\usepackage{graphicx}
\usepackage{natbib}
\usepackage{caption}
\DeclareCaptionStyle{ruled}{labelfont=normalfont,labelsep=colon,strut=off}
\usepackage{amsthm}
\usepackage{amsmath}
\usepackage{footmisc}
\usepackage[ruled,vlined]{algorithm2e}
\usepackage{bm}
\usepackage{multirow}
\usepackage{tikz}
\usepackage{tablefootnote}

\newtheorem{theorem}{Theorem}[section]
\newtheorem{lemma}[theorem]{Lemma}
\newtheorem{example}[theorem]{Example}

\newcommand{\algname}{\texttt{CSALE}}
\newcommand{\clucb}{\texttt{CLUCB-PAC}}
\newcommand{\gap}[1]{\Delta(#1) }
\newcommand{\egap}[1]{\hat{\Delta}(#1) }
\newcommand{\score}{\mu}
\newcommand{\escore}{\hat{\mu}}
\newcommand{\superset}{\cM}
\newcommand{\elim}{\mathrm{elim}}
\newcommand{\optarms}{A^*}
\newcommand{\act}{\cA}
\newcommand{\accept}{\mathrm{Acc}}
\newcommand{\event}{\eta}
\newcommand{\link}{\url{https://github.com/noabdavid/csale}}

\title{A Fast Algorithm for PAC Combinatorial Pure Exploration}
\author{
	Noa Ben-David and Sivan Sabato
}
\affiliations{
	Department of Computer Science\\
	Ben-Gurion Univesity of the Negev\\
	Beer Sheva, Israel\\
	bendanoa@post.bgu.ac.il,
	sabatos@cs.bgu.ac.il
	
}

\usepackage{bibentry}

\begin{document}

\maketitle

\begin{abstract}
We consider the problem of Combinatorial Pure Exploration (CPE), which deals with finding a combinatorial set or arms with a high reward, when the rewards of individual arms are unknown in advance and must be estimated using arm pulls. Previous algorithms for this problem, while obtaining sample complexity reductions in many cases, are highly computationally intensive, thus making them impractical even for mildly large problems. In this work, we propose a new CPE algorithm in the PAC setting, which is computationally light weight, and so can easily be applied to problems with tens of thousands of arms. This is achieved since the proposed algorithm requires a very small number of combinatorial oracle calls. The algorithm is based on successive acceptance of arms, along with elimination which is based on the combinatorial structure of the problem. We provide sample complexity guarantees for our algorithm, and demonstrate in experiments its usefulness on large problems, whereas previous algorithms are impractical to run on problems of even a few dozen arms. The code for the algorithms and experiments is provided at \link.
\end{abstract}

\section{Introduction}\label{sec:intro}
\noindent Combinatorial pure exploration (CPE) is an important statistical framework,
used in diverse applications such as channel selection in a wireless network
\citep{XueZhMaWuZh18}, job applicant screening \citep{SchumannLaFoDi19}, and
robot decision making \citep{MarcotteWaMeOl20}. In the
CPE framework, there are $n$ arms that can be pulled, where each arm is
associated with an unknown distribution of real-valued rewards. Whenever an
arm is pulled, an instantaneous reward is independently drawn from that arm's
reward distribution. An \emph{arm set} is a set of arms, and its reward is the sum of rewards of the arms it contains. The goal is to find a valid arm set with a high reward. The set of valid arm sets is called the \emph{decision class}, and it is  typically the set of legal solutions of a combinatorial problem.   The algorithm uses individual arm pulls (samples) to collect information on the rewards, and attempts to select from the decision class an arm set with a high expected reward, using a low sample complexity.

Many applications can be mapped to the CPE setting. As an example, consider the problem of finding the shortest routing path in a network \citep{TalebiZoCoPrJo17}. The network is modeled as a directed graph, where each edge represents a link between two routing servers, and the cost of routing through the link is the (random) delay in the link. The task is to find the path from a source to a target with the smallest expected delay. In this problem, the CPE arms are mapped to the links (graph edges), and pulling an arm is done by sending a packet through the corresponding link and measuring its delay. The decision class is thus the set of possible paths between the source and the target.  
As another example, consider pairing online players for matches, based on their skill compatibility. This can be represented by a maximum weight matching problem on a full graph, where the reward measures the compatibility of the players, and it can be sampled by running a simulation of a match.

As evident in the examples above, in many natural cases the size of the decision class of a CPE problem is exponential in the number of arms. Algorithms for the general CPE problem with such decision classes thus usually assume access to an \emph{oracle} which can efficiently solve the combinatorial optimization problem (see, e.g., \citealt{ChenLiKiLyCh14, GabillonLaGhOrBa16, ChenGuLi17, CaoKr19}). Using such an oracle, these algorithms are efficient, in that their computational complexity is polynomial in the problem parameters. However, while theoretically efficient, running the oracle is usually highly demanding in terms of computation resources for large problems. In existing CPE algorithms, the number of oracle calls is either similar to the sample complexity or a high degree polynomial. As a result, they are computationally heavy, and impractical to run even on moderately large problems.

In this work, we propose a new CPE algorithm that is significantly less computationally demanding than previous algorithms, and so it is practical to run on large problems. At the same time, it has sample complexity guarantees that show it can be useful for many problems. The algorithm, called \algname\ (Combinatorial Successive Acceptance with Light Elimination), uses a significantly smaller number of oracle calls than previous algorithms, and does not use any other computationally demanding operations. \algname\ works in the $(\epsilon,\delta)$-PAC setting, in which the algorithm is provided with an error parameter $\epsilon$ and a confidence parameter $\delta$, and finds, with a probability of $1-\delta$, an arm set from the decision class with an $\epsilon$-optimal reward. Compared with requiring that the algorithm find an optimal solution, the PAC approach allows convergence even if the optimal solution is not unique, and is also more suitable for cases where the gap between the (unique) optimal solution and the second-best solution is very small.

The number of oracle calls used by \algname\ is only $O(d\log(d))$, where $d$ is the maximal cardinality of an arm set. This is achieved by avoiding the common approach of searching for arms to eliminate by estimating each arm's contribution to the solution value (its \emph{gap}). \algname\ successively \emph{accepts} arms if they can be safely added to the output solution. It eliminates arms only due to the combinatorial structure of the problem. For instance, in the shortest path problem, arms that create a cycle with accepted arms can be safely eliminated. 
We provide sample complexity guarantees for \algname, which show when it can provide useful sample complexity reductions. 

We demonstrate the practical computational advantage of \algname\ by running it on instances of the $s$-$t$ shortest path problem and the maximum weight matching problem, which we generated based on real-world graph data sets.
As a baseline, we compare to the \clucb\ algorithm \citep{ChenLiKiLyCh14}, which is the only existing PAC algorithm for the general CPE problem. We show that in small graph problems, with about a dozen nodes and up to a few dozen arms (edges),
\clucb\  usually performs somewhat better than \algname\ in terms of number of samples, but \algname\ runs 5 orders of magnitude faster. We note that existing best-arm-set CPE algorithms, even if they could be adapted to the PAC setting, are even more computationally demanding than \clucb. Due to its computational requirements, we could not test \clucb\ on graphs with more than a few dozen nodes and edges. 

We run \algname\ on problems as large as thousands of nodes and tens of thousands of edges (arms) and compare it to a naive baseline, which pulls each arm the same number of times. We show that \algname\ provides sample complexity improvements compared to this baseline, while still being practical to run in terms of computational resources, even on these large problems.

\section{Related Work}\label{sec:related}

CPE was first introduced by \citet{ChenLiKiLyCh14}. They proposed \clucb, a PAC
algorithm which calls an optimization oracle twice for every arm pull. They
also propose an algorithm in the fixed-budget setting (which does not give PAC
guarantees), that relies on a \emph{constrained oracle}: this is an oracle
that outputs the optimal arm set subject to problem-specific constraints. They
observe that such an oracle exists whenever an unconstrained oracle exists,
since the former can be implemented by calling the latter with a simple
transformation of the inputs. Other PAC algorithms have been proposed for
special cases of CPE, such as top-$k$ arm selection
\citep{KalyanakrishnanTeAuSt12, ZhouChLi14, ChaudhuriKa19}, matroids
\citep{ChenGuLi16}, and dueling bandits \citep{ChenDuHuZh20}.

CPE algorithms for the fixed-confidence setting were proposed in \citet{ChenLiKiLyCh14, GabillonLaGhOrBa16, ChenGuLi17, CaoKr19}. In these algorithms, the number of oracle calls is either similar to the sample complexity or is a high-degree polynomial. 
We note that none of the algorithms mentioned above have been empirically evaluated in the works above, which focus on theoretical guarantees. As far as we know, they also do not have available implementations.

The idea of successive acceptance with limited elimination was previously proposed in a different context, of active structure learning of Bayesian networks \citep{BendavidSa21}. That problem also involves sampling for the purpose of combinatorial selection, but it is not a CPE problem. Moreover, the algorithm proposed in that work is specifically tailored to the challenges of Bayesian network learning, and relies on computationally demanding operations. In contrast, the current work proposes a computationally light algorithm for the CPE setting.

\section{Setting and Notation}\label{sec:prel}

For an integer $i$, denote $[i] := \{1,\ldots,i\}$. Let $n$ be the number of arms, and denote the set of arms by $[n]$. Each arm is associated with a reward distribution. For simplicity, we assume that the support of all reward distributions is in $[0,1]$. The results can easily be generalized to sub-Gaussian distributions.  Denote the expected reward of an arm $a\in [n]$ by $\mu_a$. Let $\bm{\mu}$ be the vector of expected rewards of the arms in $[n]$. The reward of an arm set $M \subseteq [n]$ is the sum of the expected rewards of its arms, denoted $\score(M):= \sum_{a\in M}\score_a.$ 
Denote the decision class of potential arm sets by $\superset\subseteq2^{[n]}$.  
Denote  the optimal score of an arm set in $\superset$ by $\score^* := \max_{M \in \superset} \score(M)$. Given $\epsilon>0,\delta \in (0,1)$, our goal is to find an arm set $\hat{M} \in \superset$ such that with a probability of at least $1-\delta$, $\score(M) \geq \score^* - \epsilon$. Denote the maximal cardinality among the arm sets in $\superset$ by $d(\superset):=\max_{M\in\superset} |M|$. Whenever it is clear from context, we use $d := d(\superset)$.

We assume access to a constrained oracle for the given combinatorial problem. Given a decision class, a constrained oracle $\superset$ accepts a reward vector $\bm{\mu}$, a set of required arms $A \subseteq [n]$, and a set of forbidden arms $B \subseteq [n]$, and returns an optimal solution in the decision class that satisfies these constraints. Formally, the constrained oracle is a function $\mathsf{COracle} : \reals^n\times2^{[n]}\times2^{[n]}\rightarrow \superset$  such that $\mathsf{COracle}(\bm{\mu},A,B) \in \argmax \{\mu(M) \mid M \in \superset, A\subseteq M,B\cap M = \emptyset\}$.

Denote the set of all optimal arm sets in $\superset$  by $\superset^*:= \{M\in\superset\mid\score(M)=\score^*\}.$ 
The \emph{gap} of an arm $a \in [n]$ is defined as:
\[\gap{a}:=
\begin{cases}
\score^* - \max_{M:a\notin M}\score(M), & a\in\bigcup_{M^*\in\superset^*}M^*\\
\score^* - \max_{M:a\in M}\score(M), & a\notin\bigcup_{M^*\in\superset^*}M^*
\end{cases}.\]
Note that if $a$ is included in some, but not all, of the optimal arm sets, then $\gap{a} = 0$. 

\section{The \algname\ Algorithm}\label{sec:newalg}

In this section, we present \algname, the proposed algorithm. The main idea of the algorithm is to perform successive acceptance, and to use the combinatorial structure of the decision class for elimination, thus considerably restricting the number of required oracle calls. The algorithm starts with a large accuracy threshold, and iteratively attempts to accept arms while decreasing the threshold. Arms are accepted if they belong to the current empirically optimal solution (under the constraint that it must include all arms that have already been accepted) and have a large gap $\gap{a}$. This is inspired by the acceptance criterion of graph sub-structures, proposed in \cite{BendavidSa21} in the context of Bayesian structure learning. 

The gap $\gap{a}$ is estimated by finding an empirically optimal constrained solution that does not include $a$. 
If the gap is sufficiently large, the arm is accepted. In the final round, the accuracy parameter is set to guarantee the true desired accuracy $\epsilon$, and the output arm-set is calculated, by estimating the rewards of the remaining candidate arms and maximizing the empirical score subject to all arms accepted so far.

To obtain a significant reduction in sample complexity compared to a naive uniform sampling baseline, it is helpful to eliminate arms from the list of candidates. However, as discussed above, directly estimating the gaps of arms to eliminate involves computationally demanding operations such as a large number of oracle calls. We avoid such direct estimation, and instead use the structure of the decision class itself. In many cases, this structure allows eliminating arms as a result of accepting some other arms. For instance, in the shortest path problem, edges (arms) that complete a cycle with accepted edges can be eliminated.

Formally, we define an \emph{elimination function}, denoted $\elim_n:2^{[n]}\rightarrow2^{[n]}$, which gets as input the set of arms accepted so far, and outputs the set of arms that can be eliminated as candidates given the accepted arms. This function is implemented with respect to the specific decision class. Formally, we require
\[
\elim_n(A)  \subseteq [n]\setminus \bigcup_{M\in \superset, A \subseteq  M} M.
\]
The function $\elim_n$ is used by \algname\ to eliminate candidate arms after each acceptance.
Note that $\elim_n$ is not required to always identify all possible eliminations, since this may be a computationally difficult task.

\algname\ is listed in \myalgref{active_algorithm}. It gets the confidence parameter $\delta$ and the required accuracy level $\epsilon$.
For $\tilde{\epsilon},\tilde{\delta}\in(0,1)$, we denote by $N(\tilde{\epsilon},\tilde{\delta}):=\lceil\log(2/\tilde{\delta})/(2\tilde{\epsilon}^2)\rceil$ the number of samples required by Hoeffding's inequality for estimating the expectation of random variable with support $[0,1]$, with a probability of at least $1-\tilde{\delta}$ and an error of no more than $\tilde{\epsilon}$. 
\algname\ maintains a set of \emph{active arms} $\act$. This is the set of arms that have not been accepted so far, but also have not been precluded from participating in the output solution. The set of accepted arms by iteration $j$ is denoted $\accept_j$.

\begin{algorithm}[t]
	\DontPrintSemicolon
	\KwIn{$\delta \in (0,1)$; $\epsilon > 0$. }
	\KwOut{An arm set $M\in \superset$}
	\nl \textbf{Initialize}: $\act\leftarrow[n]$, $\accept_1 \leftarrow \emptyset$, $N_0 \leftarrow 0$, $t\leftarrow1$, $j \leftarrow 1$, $\superset_1\leftarrow\superset$, $\epsilon_1\leftarrow\epsilon$, $\theta_1\leftarrow d(\superset_1)\cdot\epsilon_1$, $T \leftarrow \lceil \log_2(d) \rceil+1$.\;\label{init}
	\nl \While{$\epsilon_t > \epsilon/(d(\superset_j)-|\accept_j|)$}{
		
		\nl $N_t \leftarrow N(\epsilon_t/2, \delta/(T|\act|))$\;
		\nl Pull each arm in $\act$ for $N_t- N_{t-1}$ times\;
		\nl Update $\bm{\escore}_t$ based on the results of the arm pulls\;
		\nl  $\hat{M}_t\leftarrow\mathsf{COracle}(\bm{\escore}_t,\accept_j,\emptyset)$\;

		\nl\ForEach{$a\in \hat{M}_t\cap \act$}{
			\nl $\widetilde{M}^a_t\leftarrow\mathsf{COracle}(\bm{\escore}_t,\accept_j,\{a\})$\; \label{Mtilde}
			\nl $\egap{a}:=\bm{\escore}_t(\hat{M}_t)-\bm{\escore}_t(\widetilde{M}^a_t)$ \;
		}
		\nl\label{iteration}\While{$\exists a\in\hat{M}_t\cap \act$ such that $\egap{a}>\theta_j$\label{cond}}{
			\nl Let $a$ be some  arm that satisfies the condition \;
			\nl $\accept_{j+1}\leftarrow\accept_j\cup\{a\}$ \hspace{1em} \# accept arm $a$\;
			\nl $\act \leftarrow \act \setminus  \elim_n(\accept_{j+1})$ \label{elimination}  \;
			\nl $\superset_{j+1}\leftarrow\{M\in\superset\mid \accept_{j+1}\subseteq M\}$ \;
			\nl $\theta_{j+1}\leftarrow(d(\superset_{j+1})-|\accept_{j+1}|)\cdot\epsilon_t$  \;\label{theta_j} 
			
			\nl $j \leftarrow j+1$\;	
		}

		\nl \textbf{If } $\accept_j\in\superset_j$   
		\textbf{then} \Return $\hat{M} := \accept_j$. \;
		
		\nl $t\leftarrow t+1$ \;
		\nl $\epsilon_t \leftarrow \epsilon_{t-1} / 2; \theta_{j}\leftarrow \theta_j / 2.$ \;
		
	} 
	\nl $\epsilon_{\mathrm{last}} \leftarrow \epsilon/(d(\superset_j)-|\accept_j|)$\;
	\nl $N_T \leftarrow N(\epsilon_{\mathrm{last}}/2, \delta/(T|\act|))$ \label{step:epslast}\;
	
	\nl Pull arms in $\act$ for $N_T- N_{T-1}$ times; update $\bm{\escore}_T$.\;
	\nl \textbf{Return} $\hat{M}\leftarrow\mathsf{COracle}(\bm{\escore}_T,\accept_j,\emptyset)$. 
	\caption{\algname: Combinatorial Successive Acceptance with Light Elimination}\label{alg:active_algorithm}
\end{algorithm}

The algorithm works in rounds. At each round $t$, it pulls arms a sufficient number of times to obtain the required accuracy for that round, denoted $\epsilon_t$. Denote by $\bm\escore_t$ the vector of empirical estimates of expected arm rewards, based on the samples observed until round $t$.  
Calling $\mathsf{COracle}$ with $\bm\escore_t$, \algname\ calculates $\hat{M}_t$, the empirically optimal arm set at round $t$, out of the arm sets that are consistent with the arms that have been accepted so far. 
It then calculates the empirical gap of each active arm in $\hat{M}_t$, by calling $\mathsf{COracle}$ again, each time forbidding one of these arms.

In the next stage, every arm with an empirical gap above a calculated threshold $\theta_j$ is accepted. We show in the analysis that using this threshold, only arms in the optimal solution are accepted. Once an arm is accepted, it is removed from $\act$ along with all the arms that can be eliminated using $\elim_n$. We note that when calling $\mathsf{COracle}$, it is not necessary to explicitly ban the eliminated arms, since by definition, only arms that cannot share an arm-set with the accepted arms are eliminated. The rounds stop when the accuracy $\epsilon_t$ is sufficiently small. Note that the total number of rounds is at most $T-1 =  \lceil \log_2(d) \rceil$, where $T$ is defined in line \ref{init} in \algname.
After the rounds stop, if any arms are still active, they are pulled to obtain a final batch of samples. \algname\ then returns an arm set that maximizes the empirical reward, subject to the constraints set by the arms accepted so far.

\paragraph{Number of oracle calls.} It can be easily seen that except for
oracle calls, \algname\ does not involve any computationally demanding
operations. Therefore, the number of oracle calls that it makes is a good
proxy for its computational burden. In the worst case, \algname\ calls the
oracle once for every arm in $\hat{M}_t$, in every round $t \in
[T-1]$. Therefore, the total number of oracle calls in \algname\ is
$O(dT) = O(d\log(d))$.  This is contrasted with previous algorithms, in which
the number of oracle calls is either similar to the sample complexity, or is a
high degree polynomial. For instance, in \clucb, the number of
oracle calls is of order $\Omega(d^2 n/\epsilon^2)$ in the worst case.
In the next section, we provide the analysis of \algname.

\section{Analysis} \label{sec:analysis}

In this section, we provide correctness and sample complexity analysis for \algname. First, we define a uniform convergence event which will be used in the analysis. Denote by $J$ the total number of acceptance iterations (see line \ref{cond} in \myalgref{active_algorithm}) during the entire run of the algorithm. 
For any $j \in [J]$, let $t(j)$ be the round in which iteration $j$ occurred. Let $\theta_j$ be as defined in line \ref{theta_j}, and $\theta_J:= (d(\superset_J) - |\accept_J|)\cdot \epsilon_{\mathrm{last}}.$		Let $\eta$ be the event such that $\forall j\in[J],
\forall M\in\superset_j$, 
\begin{equation*}
|\escore_{t(j)}(M\setminus\accept_j)-\score(M\setminus\accept_j)|\leq\theta_j/2.
\end{equation*}
By a standard argument (see \lemref{event} in \appref{correctness}), $\eta$ holds with a probability at least $1-\delta$. 

To show that \algname\ is indeed a PAC algorithm, we prove the following theorem. The proof is provided in \appref{correctness}. 
\begin{theorem}\label{thm:correct}
	With probability at least $1-\delta$, the arm set $\hat{M}$ returned by \algname\ satisfies $\score(\hat{M}) \geq \score^* - \epsilon.$
\end{theorem}

We now study the sample complexity of \algname.
\algname\ can reduce the number of samples compared to a naive uniform sampling solution if it accepts arms early, before the last sampling batch. 
An even larger reduction can be obtained if the elimination function $\elim_n$ for the given decision class leads to many arm eliminations. For $k\in\nats$ such that $k\leq n$, define the \emph{elimination measure} of $\elim_n$ for a set of size $k$ to be 
$Q(n,k) := k+\min_{A\subseteq[n]:|A|=k}|\elim_n(A)|$.
Note that in many natural problems and elimination functions, $Q(n,k)$ grows with $n$. For instance, for maximum weight matching in a graph $G=(V,E)$, a natural elimination function would be $\elim_n(E'):=\bigcup_{(u,v)\in E'}\{e\in E\mid v\text{ or } u \text{ are nodes of } e \}$. Recall that in this problem, $n = |E|$. Thus, if $G$ is a full graph, we get $Q(n,k) = \Theta(k\sqrt{n})$.

Denote the set of optimal arm sets in $\superset_j$ by $\superset_j^* := \superset^* \cap \superset_j$, and the set of arms that are shared by all optimal solutions by $\optarms:=\bigcap_{M\in\superset^*}M.$
As shown in \lemref{arm_correctness} in \appref{correctness}, only arms in $\optarms$ may be accepted early. For any $\beta>0$ and $\square \in \{>,\geq,<,\leq\}$ let $\optarms_{\square\beta} = \{ a \in \optarms \mid \Delta(a) \square \beta\}$.  

\begin{theorem}\label{thm:sc}
	Let $Q(n,\cdot)$ be the elimination measure of $\elim_n$.
	For $t\in [T-1]$, let $\epsilon_t=2\epsilon/2^t$ be the accuracy level in round $t$ of \algname, and define $k_t:=|\optarms_{\geq2d\epsilon_t}\setminus\optarms_{\geq2d\epsilon_{t-1}}|$, where $\epsilon_0:=\infty$. Define $\bar{k}_t := \sum_{i\in[t]} k_t$. Further define $\Delta_t:=\max_{j\in \optarms_{<2(d-\bar{k}_{t-1})\epsilon_{t-1}}}\gap{j}. $   
	The maximal number of samples required by \algname\ is 
	
	\begin{equation}
	\begin{split}
	&\widetilde{O}\Big(\sum_{t=1}^{T-1}\frac{(d-\bar{k}_{t-1})^2Q(n,k_t)}{\Delta_t^2} \\
	&\qquad+ \frac{(d-\bar{k}_{T-1})^2 (n-Q(n,\bar{k}_{T-1}))}{\epsilon^2} \Big),
	\end{split}
	\label{eq:sc}
	\end{equation}
	where $\widetilde{O}$ suppresses logarithmic factors.
	In addition,
	\begin{enumerate}
		\item For every $t \in [T-1]$, if $Q(n,k_t)>0,$ then  $\Delta_t>2\epsilon$;  \label{small_gap}
		\item  $\bar{k}_{T-1}=|\optarms_{\geq4\epsilon}|$.\label{minimal_acc}
	\end{enumerate} 
	
\end{theorem}

Before proving the theorem, we discuss the meaning of these sample complexity guarantees. First, compare these guarantees to a naive uniform sampling algorithm, which simply pulls each arm $N(\epsilon/d,\delta/n)$ times and selects the empirically-best arm set. Such an algorithm requires $\tilde{\Theta}(d^2n\log(1/\delta)/\epsilon^2)$ samples, similarly to the worst case of \algname. This is a baseline requirement which ensures that using \algname\ over naive sampling is never harmful.
Like all CPE algorithms, the improvement of \algname\ over the naive baseline is instance dependent.

We now compare the guarantees of \algname\ to those of \clucb\ \citep{ChenLiKiLyCh14}, which is the only previous CPE algorithm for the PAC setting. The guarantees of the two algorithms are not directly comparable, since each algorithm has a different type of instance dependence.  We now give an example of a simple case in which \algname\ performs significantly better than \clucb.

\begin{example}
	\normalfont
	Consider $n$ arms, where $n > 10$ is an odd number, and let $s = (n-1)/2$. Let $\gamma > 0$. 
	For every \mbox{$a\in[s-1]$} we set $\score_a=s\gamma$. We also set $\score_s=\gamma$. All the other arms get a reward of zero. The decision class consists of three arm sets, each including exactly $s$ arms: $\{1,\ldots,s\}$, $\{1,\ldots,s-1, s+1\}$, and $\{s+2,\ldots,n\}$. Their respective rewards are $(s^2-s+1)\gamma, (s^2-s)\gamma$, and $0$.
	Thus, the reward difference between the optimal solution and the second-best solution is $\gamma$. Set $\epsilon:=\gamma$, so that the PAC algorithm is required to find the optimal solution.
	
	The worst case sample complexity of \clucb\ is $\widetilde{O}\left(\sum_{a\in A}\min\{\frac{\mathsf{width}^2(\superset)}{\gap{a}^2},\frac{d^2}{\epsilon^2}\}\right)$, where $\mathsf{width}(\superset)$ is a combinatorial property of the decision class. In our example, we have $d=s=\Theta(n)$ and $\mathsf{width}(\superset)=2s=\Theta(n).$
	In addition, for all $a \in [n]\setminus \{s,s+1\}$, we have $\gap{a} =(s^2-s+1)\gamma = \Theta(n^2\gamma ),$ while $\gap{s} = \gap{s+1} = \gamma$. 
	The sample complexity guarantee of \clucb\ is thus $O(n^2/\gamma^2)$. 
	
	We now bound the sample complexity of \algname, assuming that $\elim_n$ is maximal. We have $\optarms=[s]$. By \thmref{sc}, claim \ref{minimal_acc}, $\bar{k}_{T-1}=|\optarms_{\geq4\gamma}|=s-1.$ In addition, since all arms in 
	$\optarms_{\geq4\gamma}$ have the same gap, and $\bar{k}_{T-1}=\sum_{t \in [T-1]}k_t$ by definition, then there exists some $t_0 \in [T-1]$ for which $k_{t_0}=\bar{k}_{T-1}=s-1$, while $k_i = 0$ for $i \neq t_0$. By the definition of $Q,$ we have that $Q(n,k_{t_0})=Q(n,s-1)>0,$ and $Q(n,k_i)=0$ for $i\neq t_0.$  In addition, $\bar{k}_{t_0-1}=0.$  By \thmref{sc}, claim \ref{small_gap}, we have that $\Delta_{t_0}>2\gamma$. Since $\Delta_{t_0}$ is a gap of an arm in $\optarms$, we have $\Delta_{t_0} = \Theta(n^2\gamma ).$ From the definition of the elimination function, we have $Q(n,s-1)=s-1+s=n-2.$ Substituting these quantities in \eqref{sc}, we get that the sample complexity of \algname\ for this problem is $\widetilde{O}(1/\gamma^2)$, while the sample complexity of \clucb\ is a factor of $n^2$ greater in this example. \hfill $\qed$
\end{example}

We note that the CPE algorithms of \citet{ChenGuLi17, CaoKr19} obtain a smaller sample complexity than \clucb, by applying a different sampling and convergence analysis technique. Their sample complexity is comparable with that of \algname\ in the example above, but it can also be smaller than that of \algname. However, these algorithms do not support the PAC setting, and their sampling technique is highly computationally demanding. An interesting question for future work is whether similar sample complexity improvements can be obtained within our framework, while keeping the algorithm computationally light. 

We now turn to the proof of \thmref{sc}. First, we prove two useful lemmas.
The first lemma gives a sufficient condition for an arm to be accepted by \algname\ at a given iteration. 
\begin{lemma}\label{lem:2nd_direction}
	Assume that the event  $\event$ defined above holds for a run of \algname. Let $j$ be an iteration in the run in some round $t$ and let $\theta_j$ be as defined in \myalgref{active_algorithm}. Let $\act_j$ be the value of the set of active arms $\act$ at the start of iteration $j$. If there exists some arm in  $a \in \optarms\cap \act_j$ such that $\gap{a} > 2\theta_j$, then some arm is accepted by \algname\ at iteration $j$ in this round. 
	
\end{lemma}

\begin{proof}
	Under the assumption of the lemma, $2\theta_j < \gap{a}$. To prove the lemma, we show that $a\in\hat{M}_t$ and that $\egap{a}>\theta_j$. This implies that the condition in line \ref{cond} of \algname\ holds at iteration $j$, and so some arm must be accepted in this iteration.
	
	Let $M^*\in\superset^*$ be an optimal arm set. By definition, \mbox{$\optarms\subseteq M^*$} and so $a\in M^*$. By \lemref{arm_correctness}, $\accept_j\subseteq\optarms$. Thus,  $\superset^*_j=\superset^*$ and $M^*\in\superset_{j}$. Therefore, $\escore_t(\hat{M}_t)\geq\escore_t(M^*)$.
	For every $M\in\superset_j$ we have that $\score(M)=\score(M\setminus\accept_j)+\score(\accept_j)$, and the same holds for $\escore_{t(j)}$. Therefore, by $\event$, $	\forall M,M'\in\superset_j$, we have that
	\begin{equation}\label{eq:etatheta}
	\escore_{t(j)}(M)-\escore_{t(j)}(M')\geq\score(M)-\score(M')-\theta_j.
	\end{equation}
	In particular, this holds for $\hat{M}_t$ and $M^*$. Since we have $\escore_t(\hat{M}_t)- \escore_t(M^*) \geq 0$, it follows that $\score^* - \score(\hat{M}_t) \leq \theta_j$. Since  $\theta_j < \gap{a}= \score^* - \max_{M:a\notin M}\score(M)$, it follows that $\score(\hat{M}_t) > \max_{M:a\notin M}\score(M)$. Hence, $a \in \hat{M}_t$. 
	To prove the second part, observe that 
	\begin{align*}
	\egap{a}&=\escore_t(\hat{M}_t)-\escore_t(\widetilde{M}^a_t)
	\geq\escore_t(M^*)-\escore_t(\widetilde{M}^a_t)\\
	&\geq\score^*-\score(\widetilde{M}^a_t)-\theta_j
	\geq\gap{a}-\theta_j>\theta_j.
	\end{align*}
	In the second line, we used \eqref{etatheta}.
	This completes the proof.	
\end{proof}

The second lemma gives a condition for early acceptance of arms in \algname.
Without loss of generality, suppose that $\optarms = [v]$ for some $v \leq n$, and that the arms are ordered so that $\gap{1}\geq\gap{2}\geq\cdots\geq\gap{v}$. Note that by the definition of $\optarms$, $\gap{v} > 0$. 

\begin{lemma}\label{lem:epsilont_cond}
	For $i\in[v]$, let $\optarms_i = \{ a\in \optarms \mid \gap{a} \geq \gap{i}\}$. 
	Consider a run of \algname\ in which $\event$ holds.
	Then \algname\ accepts at least $|\optarms_i|\geq i$ arms until the end of the first round $t$ which satisfies $\epsilon_t \leq \gap{i}/(2d)$.  
\end{lemma}

\begin{proof}
	Suppose that \algname\ has accepted fewer than $|\optarms_i|$ arms until some iteration $j$. Then $\optarms_i \setminus \accept_j \neq \emptyset$. By \lemref{arm_correctness}, $\accept_j\subseteq\optarms$. Therefore, $\superset^*\subseteq\superset_j$, which means that all arms in $\optarms$ that were not accepted are still active. Hence, $\optarms_i\cap \act_j\neq\emptyset$, where $\act_j$ is as defined in \lemref{2nd_direction}. Let $a \in \optarms_i \cap \act$. By the definition of $\optarms_i$, $\gap{a}\geq\gap{i}$.
	The conditions of \lemref{2nd_direction} thus hold if $2\theta_j < \gap{i}$. This condition holds throughout the first round $t$ that satisfies $2d\epsilon_t \leq \gap{i}$. In this case, some arm will be accepted at iteration $j$.
	Round $t$ only ends when no additional arms are accepted. Therefore, at least $|\optarms_i|$ arms will be accepted until the end of this round. 
\end{proof}

Using the lemmas above, we now prove \thmref{sc}.
\begin{proof}[Proof of \thmref{sc}]
	Let $B_t$ be the number of arms that were accepted during round $t$. Consider the arms whose last round of being pulled by \algname\ is round $t$, and note that their number is at least $Q(n,B_t)$, the number of arms that are accepted or eliminated in round $t$. Therefore, the sample complexity of \algname\ is upper bounded by the expression $\sum_{t \in [T-1]}Q(n,B_t)N_{t}+(n-Q(n,B_T))N_T$, where $B_T=\sum_{t \in [T-1]}B_t$, $A_t$ is the size of the active set $\act$ at the beginning of round $t$, and 
	\[
	N_{t} \equiv N(\epsilon_{t}/2,\delta/(T\cdot A_t)) = \widetilde{O}\big(\frac{1}{\epsilon_{t}^2}(\log(n/\delta)+\log\log(d)\big).
	\]
	
	To bound this expression, observe that  for every \mbox{$t \in [T-1]$},
	letting $\alpha_t := 2(d-\bar{k}_{t-1})\epsilon_{t-1}$,
	\[
	\Delta_t:=\max_{j\in \optarms_{<\alpha_t}}\gap{j} < \alpha_t = 4(d-\bar{k}_{t-1})\epsilon_t. 
	\]
	Therefore, $\epsilon_t>\Delta_t/(4(d-\bar{k}_{t-1}))$.
	Hence, for \mbox{$t \in [T-1]$}, $N_t = \tilde{O}\big(\frac{(d-\bar{k}_{t-1})^2}{\Delta_t^2} \cdot (\log(n/\delta)+\log\log(d)\big)$.

	Next, we bound $B_t$. By the definition of $k_t$ in the theorem statement, there are $k_t$ arms in $\optarms$ with a gap between $2d\epsilon_t$ and $2d\epsilon_{t-1}$. By \lemref{epsilont_cond}, at least $k_t$ arms must be accepted by the end of round $t$. Therefore, since $N_t$ is monotonically increasing with $t$, the sample complexity can be upper bounded by 
	$\sum_{t \in [T-1]}Q(n,k_t)N_{t}+(n-Q(n,\bar{k}_{T-1}))N_T$.
	Moreover, for the last round $T$, we have $N_T = \tilde{O}\big(\frac{(d-|\accept_J|)^2}{\epsilon^2}(\log(n/\delta)+\log\log(d))\big)$, and $|\accept_J| \geq \bar{k}_{T-1}$. 
	Combining the above and suppressing logarithmic factors, we get the sample complexity upper bound:
	\begin{equation*}
	\begin{split}
	&\widetilde{O}\Big(\big(\sum_{t=1}^{T-1}\frac{(d-\bar{k}_{t-1})^2Q(n,k_t)}{\Delta_t^2} \\
	&\quad+ \frac{(d-\bar{k}_{T-1})^2 (n-Q(n,\bar{k}_{T-1}))}{\epsilon^2}\Big).
	\end{split}
	\end{equation*}
	To complete the proof, we prove claims \ref{small_gap} and \ref{minimal_acc}, which are listed in the theorem statement. 
	To prove claim \ref{small_gap}, let $t \in [T-1]$ be some round such that  $Q(n,k_t)>0$. By the definition of $Q$, it follows that $k_t>0$. This means that there exists some arm $l\in\optarms_{\geq2d\epsilon_t}\setminus\optarms_{\geq2d\epsilon_{t-1}}$. Therefore, $\epsilon_{t}\leq \gap{l}/(2d)<\epsilon_{t-1}$. Therefore, by the definition of $\Delta_t$, $\gap{l} \leq \Delta_t.$ By the condition of the main loop in \algname, we have that $\epsilon_{t}>\epsilon/d$. Combining the last two inequalities, we get that $\Delta_t>2\epsilon$. 
	To prove claim \ref{minimal_acc}, note that $\bar{k}_{T-1}=\sum_{t \in [T-1]}k_t=|\bigcup_{t \in [T-1]}\optarms_{\geq2d\epsilon_t}|=|\optarms_{\geq2d\epsilon_{T-1}}|$, and $2d\epsilon_{T-1}=2d (2\epsilon/2^{T-1}).$ Substituting $T-1 = \log_2(d)$, we get that $ 2d\epsilon_{T-1}=4\epsilon$, which completes the proof.
\end{proof}

\section{Experiments}\label{sec:experiments}

\begin{table*}
  \begin{center}
    \def\arraystretch{1.15}
		\begin{tabular}{l|lll|ll|ll}
			&
			\multicolumn{3}{c|}{Sample size $\times10^4$} 
			& \multicolumn{2}{c|}{Oracle calls} &
			\multicolumn{2}{c}{Time (millisec)}
			\\
			Experiment	& Naive &
		\clucb\  & \algname\  &  
			\clucb & \algname & \clucb & \algname\\ 
			\hline
			
			\multicolumn{8}{l}{$\epsilon=0.0625$}\\
			\hline
			path, synthetic & $339\pm0$ &$\mathbf{14}\pm9$ &25$\pm0$ &29$\pm17\times10^{4}$ &$\mathbf{2}\pm1$ &12$\pm7\times10^{3}$ &$\mathbf{0.5}\pm0$ \\
			matching, synthetic &
			$177\pm0$ &$\mathbf{16}\pm8$ &$23\pm0$ &32$\pm15\times10^{4}$ &\textbf{4}$\pm0$ &12$\pm9\times10^{4}$ &$\mathbf{2}\pm1$ \\
			path, USAir97 &
			490$\pm350$ &$\mathbf{190}\pm160$ &530$\pm470$ &37$\pm32\times10^{5}$ &$\mathbf{7}\pm2$ &13$\pm12\times10^{4}$ &$\mathbf{1}\pm0$ \\
			matching, USAir97 &
			$\mathbf{22}\pm0$ &$30\pm2$ &$\mathbf{22}\pm0$ &60$\pm3\times10^{4}$ &\textbf{1}$\pm0$ &96$\pm5\times10^{3}$ &$\mathbf{0.3}\pm0$ \\
			\hline
			\multicolumn{8}{l}{$\epsilon=0.03125$} \\
			\hline
			path, synthetic & $1355\pm0$ &$\mathbf{16}\pm10$ &99$\pm0$ &33$\pm21\times10^{4}$ &$\mathbf{2}\pm1$ &13$\pm9\times10^{3}$ &$\mathbf{0.4}\pm0$ \\
			matching, synthetic &
			$707\pm0$ &$\mathbf{19}\pm9$ &$92\pm0$ &37$\pm18\times10^{4}$ &\textbf{4}$\pm0$ &15$\pm11\times10^{4}$ &$\mathbf{2}\pm1$ \\
			path, USAir97 &
			1960$\pm1410$ &$\mathbf{900}\pm750$ &2120$\pm1880$ &18$\pm15\times10^{6}$ &$\mathbf{7}\pm2$ &72$\pm63\times10^{4}$ &$\mathbf{1}\pm1$ \\
			matching, USAir97 &
			$87\pm0$ &$\mathbf{52}\pm2$ &$87\pm0$ &10$\pm0\times10^{5}$ &\textbf{1}$\pm0$ &17$\pm1\times10^{4}$ &$\mathbf{0.3}\pm0$ \\
		\end{tabular}
	\end{center}
	\caption{A comparison of all three algorithms on small graphs. The synthetic and real experiments were repeated 100 and 10 times, respectively.}\label{tab:clucb}
      \end{table*}

In this section, we report experiments comparing \algname\ to \clucb\ and to the naive baseline algorithm described in \secref{analysis}. 
As noted in \secref{intro} and \secref{related}, there are no other known PAC-CPE algorithms, and algorithms for other settings are also highly computationally demanding.
All the experiments were run on a standard PC. The code for the algorithms and for the experiments below is provided at \link.

 We ran the algorithms on the two types of graph problems mentioned in \secref{intro}, $s$-$t$ shortest path and maximum weight matching in general graphs. 
All three algorithms require an optimization oracle. We used Dijkstra's algorithm as implemented in the \texttt{Dijkstar} package\footnote{\url{https://github.com/wylee/Dijkstar}} for the $s$-$t$ shortest path problem,
and the algorithm of \citet{Galil86} as implemented in the \texttt{Networkx} package \cite{HagbergSwSc08} for maximum weight matching in general graphs.
The matching oracle is computationally heavier, thus we tested the matching problem on smaller graphs than those used for the shortest path problem.
The properties and sources of all the data sets that we used are provided in \tabref{net_info} in \appref{experiments}. We used graphs with up to thousands of nodes and tens of thousands of edges for the shortest path experiments, and graphs with up to hundreds of nodes and thousands of edges for the matching experiments.
$\elim_n(E')$ for the shortest path problem was implemented to return all the edges that have the same start node or end node as an edge in $E'$. For the matching problem, it returned all the edges that share some node with some edge in $E'$.

In all of the experiments, we set $\delta=0.05$. Each experiment was repeated $10$ or $100$ times, depending on its computational requirements, as detailed in each result table below. 
In all experiments, the weight of the edge was used as the parameter of a Bernoulli distribution describing its instantaneous rewards. In synthetic or unweighted graphs, the edge weights were sampled uniformly at random out of $\{0.1,0.5,0.9\}$ in each run.
We tested a large range of $\epsilon$ values in each experiment. In all the reported experiments, all algorithms indeed found an $\epsilon$-optimal solution.

In the first set of experiments, we compared \clucb\ and \algname\ on small graphs, in terms of both sample size and computation requirements. The graph sizes were kept small due to the high computational requirements of \clucb.
We tested the following graphs:
\begin{itemize}
\item
  Shortest path, synthetic: A synthetic graph with $14$ nodes and $16$ edges, with a topology of 4 disjoint paths of length $4$ between the source and the target nodes (see illustration in \appref{experiments}).
\item
  Matching, synthetic: A full graph with $6$ nodes.
\item
  Shortest path, real: A source and a target were randomly drawn from the full graph of the USAir97 data set, and the smallest sub-tree that includes both nodes was extracted. If the number of edges between the source and the target was at least 4 and the number of nodes was at most $10$, then the resulting graph was used for one of the repetitions of the experiment. 
\item
  Matching, real: $6$ nodes were randomly drawn from the full graph of the USAir97 data set, and their sub-graph was extracted. If it included more than four edges, it was used for one of the repetitions of the experiment. 
\end{itemize}

\tabref{clucb} reports the results of the comparison with \clucb\ for the two smallest values of $\epsilon$ that were tested in these experiments. The full results for all values of $\epsilon$ are provided in \appref{experiments}.
      It can be seen that \clucb\ usually has the best sample complexity, although in most cases \algname\ also obtains a significant improvement over the Naive baseline. The important advantage of \algname\ can be seen when comparing the number of oracle calls and the total run time of the algorithm. These are usually 5 orders of magnitude larger for \clucb\ compared to \algname. Moreover, for \clucb, this number grows linearly with the number of samples. This is contrasted with \algname, in which the computational burden remains similar regardless of the sample size. For this reason, running \clucb\ on large graphs is impractical, while \algname\ can be easily used. 
      
\begin{table*}
	\begin{center}
	\begin{tabular}{cc|cc|c|cc}
			Type & Graph ID &   \multicolumn{2}{c|}{Sample size} & Sample size ratio & \multicolumn{2}{c}{\algname\ } \\
			& (see \tabref{net_info}) &  Naive &  \algname\   & (\algname$/$ Naive) & Oracle calls & Time (millisec)   \\ 
			\toprule
			
                  \multirow{12}{*}{\shortstack[c]{Shortest\\ Path}}
                             & 1 & $33\times10^{10}$ &$\mathbf{25}\pm16\times10^{10}$ & 76\% &5$\pm2$ &4$\pm2$ \\
			& 2 & $86\times10^{10}$ &$\mathbf{22}\pm28\times10^{10}$ & 26\% &5$\pm3$ &5$\pm4$ \\
			& 3 & $117\times10^{10}$ &$\mathbf{63}\pm60\times10^{10}$ & 54\% &5$\pm3$ &10$\pm5$ \\
			& 4 & $114\times10^{10}$ &$\mathbf{49}\pm56\times10^{10}$ & 43\% &4$\pm2$ &9$\pm4$ \\
			& 5 & $33\times10^{11}$ &$\mathbf{22}\pm16\times10^{11}$ & 67\% &6$\pm3$ &29$\pm12$ \\
			& 6 & $34\times10^{11}$ &$\mathbf{12}\pm15\times10^{11}$ & 35\% &4$\pm2$ &26$\pm12$ \\
			& 7 & $215\times10^{11}$ &$\mathbf{88}\pm98\times10^{11}$ & 41\% &12$\pm8$ &48$\pm30$ \\
			& 8 & $13\pm16\times10^{13}$ &$\mathbf{10}\pm17\times10^{13}$ & 77\% &13$\pm7$ &79$\pm47$ \\
			& 9 & $116\times10^{13}$ &$\mathbf{21}\pm38\times10^{13}$ & 18\% &17$\pm14$ &830$\pm626$ \\
			& 10 & $142\times10^{13}$ &$\mathbf{26}\pm48\times10^{13}$ & 18\% &17$\pm15$ &852$\pm587$ \\
			& 11  & $73\times10^{13}$ &$\mathbf{12}\pm23\times10^{13}$ & 16\% &16$\pm12$ &518$\pm333$ \\
			& 12 & $1153\times10^{12}$ &$\mathbf{90}\pm278\times10^{12}$ & 8\% &13$\pm11$ &561$\pm373$ \\
			
			\midrule
                  \multirow{6}{*}{Matching}
                             &	1 & $84\times10^{10}$ &$\mathbf{18}\times10^{10}$ & 21\% &22$\pm0$ &102$\pm7$ \\
			& 2 & $468\times10^{10}$ &$\mathbf{15}\times10^{10}$ & 3\% &18$\pm0$&204$\pm3$ \\
			& 3 & $244\times10^{11}$ &$\mathbf{66}\times10^{11}$ & 27\% &34$\pm2$ &622$\pm43$ \\
			& 4 & $22\times10^{12}$ &$\mathbf{17}\times10^{12}$ & 77\% &66$\pm2$ &756$\pm60$ \\
			& 5 & $177\times10^{12}$ &$\mathbf{57}\times10^{12}$ & 32\% &69$\pm3$ &3500$\pm100$ \\
			& 6 & $18\times10^{13}$ &$\mathbf{16}\times10^{13}$ & 89\% &85$\pm4$ &2700$\pm300$ \\
			& 7 & $317\times10^{13}$ &$\mathbf{59}\pm1\times10^{13}$ & 19\% &356$\pm7$ &28$\pm1\times10^{4}$ \\
			& 8  & $238\times10^{13}$ &$\mathbf{34}\times10^{13}$ & 14\% &307$\pm1$ &20$\pm1\times10^{4}$ \\

		\end{tabular}
              \end{center}
	\caption{Experiments with larger graphs, with $\epsilon=0.001$. Standard deviations are omitted when they are smaller than $1\%$ of the average. Shortest paths and matching experiments were repeated 100 and 10 times, respectively.}\label{tab:big_exp}
\end{table*}

Next, we tested \algname\ on larger graphs, and compared it to the Naive baseline. As mentioned above, it was impractical to run \clucb\ on these larger problems. \tabref{big_exp} reports the results for $\epsilon = 0.001$ (the smallest value of $\epsilon$ that we tested) for all the graphs that were tested for each problem. 
The results for all values of $\epsilon$ are provided in \appref{experiments}.
It can be seen in \tabref{big_exp} that the sample size required by \algname\ is significantly smaller than that of the naive baseline. 
The run time of \algname\ is small even on large networks with a very large numbers of samples.

\begin{table*}
	\begin{center}
		\begin{tabular}{c|cc|c|ccc}
			$\epsilon$ &   \multicolumn{2}{c}{Sample size} & Sample size ratio & \multicolumn{3}{c}{\algname\ }\\
			 &  Naive & \algname\   & (\algname$/$ Naive) & Oracle calls & Time (millisec) & Accepted early\\ 
			\toprule
			
			0.50000 & $29\times10^{8}$ &$\mathbf{28}\pm4\times10^{8}$ & 97\% &48$\pm14$ &13$\pm3\times10^{2}$ & 17\%$\pm$19\%\\
			 0.25000 & $12\times10^{9}$ &$\mathbf{10}\pm2\times10^{9}$ & 83\% &46$\pm14$ &12$\pm3\times10^{2}$ & 31\%$\pm$26\% \\
			0.12500 & $47\times10^{9}$ &$\mathbf{33}\pm13\times10^{9}$ & 70\% &42$\pm14$ &11$\pm3\times10^{2}$ & 51\%$\pm$33\% \\
			0.06250 & $188\times10^{9}$ &$\mathbf{94}\pm60\times10^{9}$ & 50\% &37$\pm13$ &10$\pm3\times10^{2}$ & 66\%$\pm$32\% \\
			0.03125 & $75\times10^{10}$ &$\mathbf{19}\pm21\times10^{10}$ & 25\% &30$\pm13$ &832$\pm296$ & 85\%$\pm$28\% \\
			 0.01000 & $73\times10^{11}$ &$\mathbf{13}\pm22\times10^{11}$ & 18\% &20$\pm13$ &592$\pm295$ & 85\%$\pm$27\%\\
			 0.00500 & $294\times10^{11}$ &$\mathbf{49}\pm91\times10^{11}$ & 17\% &16$\pm12$ &494$\pm285$ & 85\%$\pm$27\%\\
			 0.00100 & $73\times10^{13}$ &$\mathbf{12}\pm23\times10^{13}$ & 16\% &16$\pm12$ &518$\pm333$& 85\%$\pm$28\% \\
		\end{tabular}
	\end{center}
	\caption{p2p-Gnutella08 network results. Each experiment was repeated 100 times.}\label{tab:11}
\end{table*}

We demonstrate the dependence on $\epsilon$ in \tabref{11}, which lists the results for all values of $\epsilon$ for the p2p-Gnutella08 network. As $\epsilon$ becomes smaller, the sample size saving by \algname\ (see the ``sample size ratio'' column) becomes more significant, up to a problem-dependent saturation point. The ``accepted early'' column indicates the fraction of the arms in the solution that were accepted early. When this ratio can no longer increase due to the problem structure, additional improvements are not possible.  

Overall, our results show that unlike previous CPE algorithms, \algname\ is practical to run on large problems, due to its light computational requirements and the independence of the number of oracle calls from the number of arm pulls. When $\epsilon$ is sufficiently small, it provides a significant reduction in the number of samples over the baseline.

\section{Discussion}\label{sec:discussion}
We propose a new algorithm for PAC-CPE. This algorithm has the advantage that it is computationally light, thus allowing it to run on large problems, while reducing the sample complexity compared to the baseline. 

\bibliography{mybib}

\appendix

\clearpage

\section{Proof of \thmref{correct}}\label{ap:correctness}
We first observe that with a high probability, throughout the run of \algname, in each iteration $j$ the empirical scores of arm sets  in $\superset_j$ (that is, arm sets that include $\accept_j$) when disregarding the scores of the arms in $\accept_j$, are close to their true scores. This is formally proved in the following lemma. The proof follows a standard uniform convergence argument, similarly to \citet{BendavidSa21}. We give below the full proof for completeness.

\begin{lemma}\label{lem:event}
	The event  $\event$ defined in \secref{analysis} holds with probability at least $1-\delta$.
\end{lemma}

\begin{proof}
	Let $\act_t$ be the set $\act$ at the beginning of round $t$. Define the following uniform convergence event:
	\[
	\tau := \{\forall t \in [T],\forall a\in \act_t, |\escore_t(a)-\score_a|\leq\epsilon_{t}/2\}.
	\]
	
	In \algname, for any arm in $\act_t$, $\escore_t(a)$ is estimated based on $N_t = N(\epsilon_t/2, \delta/(T|\act_t|))$ samples. Therefore, by the definition of $N(\cdot, \cdot)$ and Hoeffding's inequality, in every round the uniform convergence requirement holds with probability at least $1-\delta/T$ (conditioned on the previous rounds). Since there are at most $T$ rounds, $\tau$ holds with probability at least $1-\delta$.
	Under $\tau$, we have that at any iteration $j$ in a round $t$, for any arm set $M \in \superset_j$, 
	\begin{align*}
	&|\score(M\setminus \accept_j) - \escore_t(M\setminus \accept_j)| \leq \sum_{a\in M\setminus \accept_j }|\score_a - \escore_t(a)|\\ &\quad\leq (d(\superset_j)-|\accept_j|)\cdot\epsilon_t/2 = \theta_j/2.
	\end{align*}
	
	This completes the proof.
\end{proof}

We now prove that under the event $\event$, if an arm is accepted by \algname, then it must be in $A^*$.

\begin{lemma}\label{lem:arm_correctness}
	Assume  that $\event$ occurs.  Then, $\accept_J\subseteq\optarms.$
\end{lemma}

\begin{proof}
	The proof is by induction on the iteration $j$. At the beggining of iteration $j=1$, we have $\accept_1 = \emptyset$, hence $\accept_1 \subseteq \optarms$. Now, suppose that the claim holds for iteration $j$, that is, $\accept_j \subseteq \optarms$. Note that this implies that all the sets in $\superset^*$ include $\accept_{j}$, therefore,  $\superset_j^* =\superset^*$.
	
	Let $a \in [n] \setminus A^*$. Then, for some $M^*\in \superset^* = \superset^*_j$, we have $a\notin M^*$. We prove that $a$ cannot be accepted in this iteration. 
	We have $\egap{a} = \escore_t(\hat{M}_t)-\escore_t(\widetilde{M}^a_t)$. Recall that
	$\widetilde{M}^a_t = \mathsf{COracle}(\bm{\escore}_t,\accept_j,\{a\})$. Therefore, it is empirically optimal out of the arm sets that include $\accept_j$ and do not include $a$. Since $M^*$ satisfies these conditions, it follows that 
	$\escore_t(M^*)\leq\escore_t(\widetilde{M}^a_t)$. Therefore, $\egap{a} \leq \escore_t(\hat{M}_t)-\escore_t(M^*)$.  For any $M \in \superset_j$, we have $\accept_j \subseteq M$, hence $\score(M) = \score(M \setminus \accept_j) + \score(\accept_j)$, and similarly for $\escore_T$. Therefore, $\escore_t(\hat{M}_t)-\escore_t(M^*)=\escore_t(\hat{M}_t\setminus\accept_j)-\escore_t(M^*\setminus\accept_j);$
	Then, by event $\eta$, $\egap{a} \leq\score(\hat{M}_t\setminus\accept_j)-\score(M^*\setminus\accept_j)= \score(\hat{M}_t)-\score(M^*)+\theta_j \leq \theta_j$. Thus, $a$ cannot be accepted at this iteration. Therefore, the arm accepted in iteration $j$ is in $\optarms$, hence $\accept_{j+1} \subseteq \optarms$.
	This proves the claim. 
\end{proof}

We now prove the theorem.
\begin{proof}[Proof of \thmref{correct}]
	Assume that $\event$ holds. By \lemref{event}, this occurs with probability at least $1-\delta$. Under this event, for any $M \in \superset_J$, 
	\begin{align}
	&|\score(M \setminus \accept) -\escore_T(M \setminus \accept)|
	\leq\theta_J/2\nonumber\\
	&\quad= (d(\superset_J) - |\accept|)\cdot \epsilon_{\mathrm{last}}/2= \epsilon/2. \label{eq:elast}
	\end{align}

	For any $M \in \superset_J$, we have $\accept_J \subseteq M$, hence $\score(M) = \score(\accept_J) + \score(M \setminus \accept_J)$, and similarly for $\escore_T$. 
	
	By \lemref{arm_correctness}, we have that $\accept_J\subseteq\optarms$. Thus, by the definition of $\optarms$, it holds that $\superset_J^*=\superset^*$. Let $M^*$ be some arm set in $\superset^* = \superset^*_J$. Recall that also \mbox{$\hat{M} \in \superset_J$}.
	By the definition of $\hat{M}$, $\escore_T(\hat{M}) \geq \escore_T(M^*)$.
	Therefore, \mbox{$\escore_T(\hat{M} \setminus \accept_J) \geq \escore_T(M^* \setminus \accept_J)$}.
	Combining this with \eqref{elast}, it follows that
	\begin{align*}
	&\score(\hat{M} \setminus \accept_J) \geq \escore_T(\hat{M} \setminus \accept_J) - \epsilon/2 \\
	&\quad\geq \escore_T(M^* \setminus \accept_J) - \epsilon/2 \geq \score(M^* \setminus \accept_J) - \epsilon.
	\end{align*}
	Hence, $\score(\hat{M}) \geq \score(M^*) - \epsilon$, which proves the claim.
\end{proof}

\section{Full experiment results}\label{ap:experiments}
In this section, we provide the full results of all the experiments.
\figref{synthetic} illustrates the synthetic network used in the shortest-path synthetic experiments reported in \tabref{paths_synthetic}. 

Properties of all the datasets used in the experiments are reported in \tabref{net_info}. Results for the shortest path experiment with the synthetic graph in \figref{synthetic} are reported in \tabref{paths_synthetic}. Results for the matching experiment on a synthetic full graph with six nodes are reported in \tabref{matchings_synthetic}. Results for shortest path and matching on small sub-graphs of USAir97 graph are reported in \tabref{paths_small} and in \tabref{matching_small}, respectively. Results for the benchmark graphs detailed in \tabref{net_info} for shortest path are reported in Tables \ref{tab:paths_large1},\ref{tab:paths_large2}. Results for the graphs detailed in \tabref{net_info} (except for the largest graphs, 9-12) for matching are reported in Tables \ref{tab:matching_large1},\ref{tab:matching_large2}.

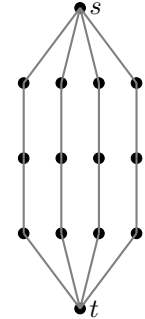
\begin{figure}[b]
	\begin{center}
		\begin{tikzpicture}
		
		\filldraw[black] (-0.25,4) circle (2pt) node[anchor=west] {$s$};
		
		\filldraw[black] (-1,3) circle (2pt) node[anchor=west] {};
		\filldraw[black] (-1,2) circle (2pt) node[anchor=west] {};
		\filldraw[black] (-1,1) circle (2pt) node[anchor=west] {};
		
		\filldraw[black] (-0.5,3) circle (2pt) node[anchor=west] {};
		\filldraw[black] (-0.5,2) circle (2pt) node[anchor=west] {};
		\filldraw[black] (-0.5,1) circle (2pt) node[anchor=west] {};
		
		\filldraw[black] (0,3) circle (2pt) node[anchor=west] {};
		\filldraw[black] (0,2) circle (2pt) node[anchor=west] {};
		\filldraw[black] (0,1) circle (2pt) node[anchor=west] {};
		
		\filldraw[black] (0.5,3) circle (2pt) node[anchor=west] {};
		\filldraw[black] (0.5,2) circle (2pt) node[anchor=west] {};
		\filldraw[black] (0.5,1) circle (2pt) node[anchor=west] {};
		
		\filldraw[black] (-0.25,0) circle (2pt) node[anchor=west] {$t$};
		
		\draw[gray, thick] (-0.25,4) -- (-1,3);
		\draw[gray, thick] (-1,3) -- (-1,2);
		\draw[gray, thick] (-1,2) -- (-1,1);
		\draw[gray, thick] (-1,1) -- (-0.25,0);
		
		\draw[gray, thick] (-0.25,4) -- (-0.5,3);
		\draw[gray, thick] (-0.5,3) -- (-0.5,2);
		\draw[gray, thick] (-0.5,2) -- (-0.5,1);
		\draw[gray, thick] (-0.5,1) -- (-0.25,0);
		
		\draw[gray, thick] (-0.25,4) -- (0,3);
		\draw[gray, thick] (0,3) -- (0,2);
		\draw[gray, thick] (0,2) -- (0,1);
		\draw[gray, thick] (0,1) -- (-0.25,0);
		
		\draw[gray, thick] (-0.25,4) -- (0.5,3);
		\draw[gray, thick] (0.5,3) -- (0.5,2);
		\draw[gray, thick] (0.5,2) -- (0.5,1);
		\draw[gray, thick] (0.5,1) -- (-0.25,0);
		\end{tikzpicture}
	\end{center}
	\caption{Synthetic graph used in the synthetic shortest-path experiment. In each experiment, the edge weights were randomly sampled out of $\{0.1,0.5,0.9\}$.}\label{fig:synthetic}
\end{figure}

\begin{table*}[h]
	\def\arraystretch{1.15}
	\begin{center}
		\resizebox{0.9\textwidth}{!}{
			\begin{tabular}{llllcc}
				Graph & No.~nodes & No.~edges & Source & Dircected? & Weighted?\\
				\toprule
				1. Southern women & 18 & 278 & \citet{DavisGaGa41}$^1$ &  \xmark& \cmark\\
				\hline
				2. Freemans EIES & 32 & 460 & \citet{Freemanfr79}$^2$ & \cmark & \cmark\\
				\hline
				3. Organisational (consult; advice) & 46 & 879 & \multirow{4}{*}{\citet{CrossPa04} $^3$} & \multirow{4}{*}{\cmark} & \multirow{4}{*}{\cmark}\\
				4. Organisational (consult; value) & 46 & 858 & & & \\
				5. Organisational (R\&D; advice) & 77 & 2228 & & & \\
				6. Organisational (R\&D; value) & 77 & 2326 & & &\\
				\hline
				7. C.elegans neural network & 306 & 2345 & \citet{WattsSt98} $^4$& \cmark& \cmark\\
				\hline
				8. USAir97 & 332 & 2100 & \citet{ColizzaPaVe07}$^5$ & \cmark & \cmark\\
				\hline
				9. p2p-Gnutella05 & 8,846 & 31,839 & \multirow{4}{*}{\shortstack[l]{\citet{RipeanuFo02},\\ \citet{LeskovecKlFa07}}$^6$} & \multirow{4}{*}{\xmark} & \multirow{4}{*}{\xmark}\\
				10. p2p-Gnutella06 & 8,717 & 31,525& & &\\
				11. p2p-Gnutella08 & 6,301 & 20,777 & & &\\
				12. p2p-Gnutella09 & 8,114 & 26,013& & &\\
				\hline
				
			\end{tabular}
		}
	\end{center}
	\captionsetup{justification=centering}
	
	\caption{Graph data sets used in our experiments.
		Urls for obtaining the data:\\
		1. \url{https://toreopsahl.com/datasets/\#southernwomen}\\
		2. \url{https://toreopsahl.com/datasets/\#FreemansEIES}\\
		3. \url{https://toreopsahl.com/datasets/\#Cross\textunderscore Parker}\\
		4. \url{https://toreopsahl.com/datasets/\#celegans}\\
		5. \url{http://networkrepository.com/USAir97.php}\\
		6. \url{https://snap.stanford.edu/data/} }\label{tab:net_info}
\end{table*}

\clearpage

\begin{table*}
	\begin{center}
		\begin{tabular}{l|lll|ll|ll}
			\toprule
			$\mathbf{\epsilon}$&
			\multicolumn{3}{c|}{Sample size}
			& \multicolumn{2}{c|}{Oracle calls} &
			\multicolumn{2}{c}{Time (millisec)}
			\\
			& Naive &
			\clucb\  & \algname\  &  
			\clucb & \algname & \clucb & \algname\\ 
			\toprule
			
			2.00000 & $\mathbf{3312}$ &$6253\pm1426$ &3872$\pm0$ &12$\pm3\times10^{3}$ &$\mathbf{11}\pm0$ &487$\pm137$ &$\mathbf{1}$ \\
			1.00000 & $\mathbf{13}\times10^{3}$ &$17\pm6\times10^{3}$ &15$\pm0\times10^{3}$ &34$\pm11\times10^{3}$ &$\mathbf{11}\pm0$ &1330$\pm472$ &$\mathbf{1}$ \\
			0.50000 & $53\times10^{3}$ &$\mathbf{40}\pm17\times10^{3}$ &60$\pm9\times10^{3}$ &80$\pm35\times10^{3}$ &$\mathbf{11}\pm1$ &3206$\pm1469$ &$\mathbf{1}$ \\
			0.25000 & $212\times10^{3}$ &$\mathbf{75}\pm39\times10^{3}$ &155$\pm100\times10^{3}$ &15$\pm8\times10^{4}$ &$\mathbf{9}\pm3$ &6064$\pm3346$ &$\mathbf{1}$ \\
			0.12500 & $85\times10^{4}$ &$\mathbf{11}\pm6\times10^{4}$ &16$\pm9\times10^{4}$ &22$\pm13\times10^{4}$ &$\mathbf{5}\pm3$ &9076$\pm5275$ &$\mathbf{1}$ \\
			0.06250 & $339\times10^{4}$ &$\mathbf{14}\pm9\times10^{4}$ &25$\pm0\times10^{4}$ &29$\pm17\times10^{4}$ &$\mathbf{2}\pm1$ &12$\pm7\times10^{3}$ &$\mathbf{0.5}$ \\
			0.03125 & $1355\times10^{4}$ &$\mathbf{16}\pm10\times10^{4}$ &99$\pm0\times10^{4}$ &33$\pm21\times10^{4}$ &$\mathbf{2}\pm1$ &13$\pm9\times10^{3}$ &$\mathbf{0.4}$ \\
			
		\end{tabular}
	\end{center}
	\caption{$s$-$t$ shortest paths synthetic graph results. Each experiment was repeated 100 times.}\label{tab:paths_synthetic}
\end{table*}

\begin{table*}
	\begin{center}
		\begin{tabular}{l|lll|ll|ll}
			\toprule
			$\mathbf{\epsilon}$&
			\multicolumn{3}{c|}{Sample size}
			& \multicolumn{2}{c}{Oracle calls} &
			\multicolumn{2}{c}{Time (millisec)}
			\\
			& Naive &
			\clucb\  & \algname\  &  
			\clucb & \algname & \clucb & \algname\\ 
			\toprule
			
			2.00000 & $\mathbf{1740}$ &$6437\pm2172$ &$2025\pm0$ &13$\pm4\times10^{3}$ &\textbf{9}$\pm0$ &4811$\pm2854$ &$\mathbf{5}\pm3$ \\
			1.00000 & $\mathbf{6915}$ &$19043\pm7317$ &$8100\pm0$ &38$\pm15\times10^{3}$ &\textbf{9}$\pm0$ &14$\pm9\times10^{3}$ &$\mathbf{5}\pm3$ \\
			0.50000 & $\mathbf{28}\times10^{3}$ &$41\pm16\times10^{3}$ &$32\pm0\times10^{3}$ &81$\pm32\times10^{3}$ &\textbf{9}$\pm0$ &31$\pm21\times10^{3}$ &$\mathbf{5}\pm3$ \\
			0.25000 & $111\times10^{3}$ &$\mathbf{79}\pm35\times10^{3}$ &104$\pm46\times10^{3}$ &16$\pm7\times10^{4}$ &$\mathbf{8}\pm2$ &61$\pm42\times10^{3}$ &$\mathbf{4}\pm3$ \\
			0.12500 & $442\times10^{3}$ &$\mathbf{124}\pm57\times10^{3}$ &65$\pm34\times10^{3}$ &25$\pm11\times10^{4}$ &$\mathbf{4}\pm1$ &94$\pm64\times10^{3}$ &$\mathbf{2}\pm1$ \\
			0.06250 & $177\times10^{4}$ &$\mathbf{16}\pm8\times10^{4}$ &$23\pm0\times10^{4}$ &32$\pm15\times10^{4}$ &\textbf{4}$\pm0$ &12$\pm9\times10^{4}$ &$\mathbf{2}\pm1$ \\
			0.03125 & $707\times10^{4}$ &$\mathbf{19}\pm9\times10^{4}$ &$92\pm0\times10^{4}$ &37$\pm18\times10^{4}$ &\textbf{4}$\pm0$ &15$\pm11\times10^{4}$ &$\mathbf{2}\pm1$ \\
			
		\end{tabular}
	\end{center}
	\caption{maximum weight matching synthetic graph results. Each experiment was repeated 100 times.}\label{tab:matchings_synthetic}
\end{table*}

\begin{table*}
	\begin{center}
		\resizebox{0.9\textwidth}{!}{
			\begin{tabular}{l|lll|ll|ll}
				\toprule
				$\mathbf{\epsilon}$&
				\multicolumn{3}{c|}{Sample size}
				& \multicolumn{2}{c}{Oracle calls} &
				\multicolumn{2}{c}{Time (millisec)}
				\\
				& Naive &
				\clucb\  & \algname\  &  
				\clucb & \algname & \clucb & \algname\\ 
				\toprule
				
				2.00000 & 4804$\pm3448$ &$\mathbf{1468}\pm1224$ &5188$\pm4581$ &2908$\pm2447$ &$\mathbf{8}\pm3$ &100$\pm91$ &$\mathbf{1}\pm0$ \\
				1.00000 & 19188$\pm13776$ &$\mathbf{6179}\pm5617$ &20735$\pm18317$ &12$\pm11\times10^{3}$ &$\mathbf{8}\pm2$ &428$\pm412$ &$\mathbf{1}\pm0$ \\
				0.50000 & 77$\pm55\times10^{3}$ &$\mathbf{24}\pm25\times10^{3}$ &83$\pm73\times10^{3}$ &48$\pm49\times10^{3}$ &$\mathbf{7}\pm2$ &1612$\pm1739$ &$\mathbf{1}\pm0$ \\
				0.25000 & 31$\pm22\times10^{4}$ &$\mathbf{11}\pm13\times10^{4}$ &33$\pm29\times10^{4}$ &22$\pm26\times10^{4}$ &$\mathbf{7}\pm2$ &7750$\pm9258$ &$\mathbf{1}\pm0$ \\
				0.12500 & 123$\pm88\times10^{4}$ &$\mathbf{45}\pm47\times10^{4}$ &133$\pm117\times10^{4}$ &90$\pm94\times10^{4}$ &$\mathbf{7}\pm2$ &31$\pm34\times10^{3}$ &$\mathbf{1}\pm0$ \\
				0.06250 & 49$\pm35\times10^{5}$ &$\mathbf{19}\pm16\times10^{5}$ &53$\pm47\times10^{5}$ &37$\pm32\times10^{5}$ &$\mathbf{7}\pm2$ &13$\pm12\times10^{4}$ &$\mathbf{1}\pm0$ \\
				0.03125 & 196$\pm141\times10^{5}$ &$\mathbf{90}\pm75\times10^{5}$ &212$\pm188\times10^{5}$ &18$\pm15\times10^{6}$ &$\mathbf{7}\pm2$ &72$\pm63\times10^{4}$ &$\mathbf{1}\pm1$ \\
				
			\end{tabular}
		}
	\end{center}
	\caption{$s$-$t$ shortest paths with sub-graphs of USAir97 of up to ten nodes. Each experiment was repeated 10 times.}\label{tab:paths_small}
\end{table*}

\begin{table*}[h]
	\begin{center}
		\resizebox{0.9\textwidth}{!}{
			\begin{tabular}{l|lll|ll|ll}
				\toprule
				$\mathbf{\epsilon}$&
				\multicolumn{3}{c|}{Sample size}
				& \multicolumn{2}{c}{Oracle calls} &
				\multicolumn{2}{c}{Time (millisec)}
				\\
				& Naive &
				\clucb\  & \algname\  &  
				\clucb & \algname & \clucb & \algname\\ 
				\toprule
				
				2.00000 & $\mathbf{215}$ &$810\pm48$ &$\mathbf{215}$ &1612$\pm96$ &\textbf{1} &216$\pm14$ &$\mathbf{0.2}$ \\
				1.00000 & $\mathbf{850}$ &$3396\pm190$ &$\mathbf{850}$ &6784$\pm379$ &\textbf{1} &901$\pm52$ &$\mathbf{0.2}$ \\
				0.50000 & $3395$ &$\mathbf{13658}\pm696$ &$3395$ &27$\pm1\times10^{3}$ &\textbf{1} &3670$\pm285$ &$\mathbf{0.2}$ \\
				0.25000 & $\mathbf{14}\times10^{3}$ &$49\pm2\times10^{3}$ &$\mathbf{14}\times10^{3}$ &97$\pm5\times10^{3}$ &\textbf{1} &14$\pm1\times10^{3}$ &$\mathbf{0.2}$ \\
				0.12500 & $\mathbf{54}\times10^{3}$ &$136\pm7\times10^{3}$ &$\mathbf{54}\times10^{3}$ &27$\pm1\times10^{4}$ &\textbf{1} &43$\pm3\times10^{3}$ &$\mathbf{0.2}$ \\
				0.06250 & $\mathbf{22}\times10^{4}$ &$30\pm2\times10^{4}$ &$\mathbf{22}\times10^{4}$ &60$\pm3\times10^{4}$ &\textbf{1} &96$\pm5\times10^{3}$ &$\mathbf{0.3}$ \\
				0.03125 & $87\times10^{4}$ &$\mathbf{52}\pm2\times10^{4}$ &$87\times10^{4}$ &10$\pm0\times10^{5}$ &\textbf{1} &17$\pm1\times10^{4}$ &$\mathbf{0.3}$ \\
			\end{tabular}
		}
	\end{center}
	\caption{Maximum weight matchings with sub-graphs of USAir97 with six nodes. Each experiment was repeated 10 times.}
	\label{tab:matching_small}
\end{table*}

\clearpage
\onecolumn
\begin{table}
	\begin{center}
		\begin{tabular}{cc|cc|c|ccc}
			\toprule
			Graph & $\epsilon$ &   \multicolumn{2}{c|}{Sample size} & Sample size ratio & \multicolumn{3}{c}{\algname\ }\\
			ID & &  Naive & \algname\   & (\algname$/$ Naive) & Oracle calls & Time (millisec) & Accepted early\\ 
			\toprule
			
			\multirow{9}{*}{1}& 1.00000 & $\mathbf{33}\times10^{4}$ &$37\times10^{4}$ & 112\% &6$\pm1$ &4$\pm0$ &0\%\\
			& 0.50000 & $\mathbf{13}\times10^{5}$ &$15\times10^{5}$ & 115\% &6$\pm1$ &4$\pm0$ &0\%\\
			& 0.25000 & $\mathbf{53}\times10^{5}$ &$59\times10^{5}$ & 111\% &6$\pm1$ &4$\pm0$ &0\%\\
			& 0.12500 & $\mathbf{21}\times10^{6}$ &$24\times10^{6}$ & 114\% &6$\pm1$ &4$\pm2$ &0\%\\
			& 0.06250 & $85\times10^{6}$ &$\mathbf{82}\pm32\times10^{6}$ & 96\% &6$\pm2$ &4$\pm1$ &15\%$\pm$36\%\\
			& 0.03125 & $34\times10^{7}$ &$\mathbf{26}\pm17\times10^{7}$ & 76\% &5$\pm2$ &3$\pm1$ &34\%$\pm$47\%\\
			& 0.01000 & $33\times10^{8}$ &$\mathbf{25}\pm16\times10^{8}$ & 76\% &5$\pm2$ &3$\pm1$ &34\%$\pm$47\%\\
			& 0.00500 & $13\times10^{9}$ &$\mathbf{10}\pm7\times10^{9}$ & 77\% &5$\pm2$ &3$\pm1$ &34\%$\pm$47\%\\
			& 0.00100 & $33\times10^{10}$ &$\mathbf{25}\pm16\times10^{10}$ & 76\% &5$\pm2$ &4$\pm2$ &34\%$\pm$47\%\\
			\midrule\multirow{9}{*}{2}& 1.00000 & $\mathbf{86}\times10^{4}$ &$98\times10^{4}$ & 114\% &12$\pm3$ &9$\pm1$ &0\%\\
			& 0.50000 & $\mathbf{34}\times10^{5}$ &$39\times10^{5}$ & 115\% &11$\pm3$ &9$\pm1$ &0\%\\
			& 0.25000 & $\mathbf{14}\times10^{6}$ &$16\times10^{6}$ & 114\% &11$\pm3$ &9$\pm1$ &0\%\\
			& 0.12500 & $\mathbf{55}\times10^{6}$ &$63\times10^{6}$ & 115\% &11$\pm3$ &9$\pm1$ &0\%\\
			& 0.06250 & $\mathbf{22}\times10^{7}$ &$25\times10^{7}$ & 114\% &11$\pm3$ &9$\pm3$ &0\%\\
			& 0.03125 & $\mathbf{88}\times10^{7}$ &100$\pm5\times10^{7}$ & 114\% &11$\pm3$ &9$\pm1$ &2\%$\pm$14\%\\
			& 0.01000 & $\mathbf{86}\times10^{8}$ &92$\pm17\times10^{8}$ & 107\% &10$\pm3$ &9$\pm2$ &9\%$\pm$25\%\\
			& 0.00500 & $34\times10^{9}$ &$\mathbf{31}\pm9\times10^{9}$ & 91\% &10$\pm3$ &8$\pm2$ &39\%$\pm$43\%\\
			& 0.00100 & $86\times10^{10}$ &$\mathbf{22}\pm28\times10^{10}$ & 26\% &5$\pm3$ &5$\pm4$ &86\%$\pm$30\%\\
			\midrule\multirow{9}{*}{3}& 1.00000 & $\mathbf{12}\times10^{5}$ &$13\times10^{5}$ & 108\% &7$\pm1$ &14$\pm2$ &0\%\\
			& 0.50000 & $\mathbf{47}\times10^{5}$ &$52\times10^{5}$ & 111\% &7$\pm1$ &14$\pm2$ &0\%\\
			& 0.25000 & $\mathbf{19}\times10^{6}$ &$21\times10^{6}$ & 111\% &7$\pm1$ &14$\pm3$ &0\%\\
			& 0.12500 & $\mathbf{75}\times10^{6}$ &76$\pm19\times10^{6}$ & 101\% &7$\pm2$ &14$\pm2$ &10\%$\pm$30\%\\
			& 0.06250 & $30\times10^{7}$ &$\mathbf{19}\pm13\times10^{7}$ & 63\% &6$\pm2$ &12$\pm4$ &55\%$\pm$49\%\\
			& 0.03125 & $120\times10^{7}$ &$\mathbf{64}\pm61\times10^{7}$ & 53\% &5$\pm3$ &11$\pm10$ &55\%$\pm$49\%\\
			& 0.01000 & $117\times10^{8}$ &$\mathbf{63}\pm60\times10^{8}$ & 54\% &5$\pm3$ &10$\pm5$ &55\%$\pm$49\%\\
			& 0.00500 & $47\times10^{9}$ &$\mathbf{25}\pm24\times10^{9}$ & 53\% &5$\pm3$ &10$\pm5$ &55\%$\pm$49\%\\
			& 0.00100 & $117\times10^{10}$ &$\mathbf{63}\pm60\times10^{10}$ & 54\% &5$\pm3$ &10$\pm5$ &55\%$\pm$49\%\\
			\midrule\multirow{9}{*}{4}& 1.00000 & $\mathbf{11}\times10^{5}$ &13$\pm1\times10^{5}$ & 118\% &6$\pm1$ &13$\pm2$ &1\%$\pm$6\%\\
			& 0.50000 & $\mathbf{46}\times10^{5}$ &50$\pm3\times10^{5}$ & 109\% &6$\pm1$ &14$\pm2$ &1\%$\pm$6\%\\
			& 0.25000 & $\mathbf{18}\times10^{6}$ &20$\pm3\times10^{6}$ & 111\% &6$\pm1$ &14$\pm2$ &4\%$\pm$18\%\\
			& 0.12500 & $73\times10^{6}$ &$\mathbf{55}\pm33\times10^{6}$ & 75\% &6$\pm2$ &12$\pm4$ &38\%$\pm$48\%\\
			& 0.06250 & $29\times10^{7}$ &$\mathbf{14}\pm13\times10^{7}$ & 48\% &5$\pm2$ &10$\pm4$ &65\%$\pm$47\%\\
			& 0.03125 & $117\times10^{7}$ &$\mathbf{51}\pm57\times10^{7}$ & 44\% &4$\pm2$ &9$\pm5$ &65\%$\pm$47\%\\
			& 0.01000 & $114\times10^{8}$ &$\mathbf{49}\pm56\times10^{8}$ & 43\% &4$\pm2$ &8$\pm4$ &65\%$\pm$47\%\\
			& 0.00500 & $46\times10^{9}$ &$\mathbf{20}\pm22\times10^{9}$ & 43\% &4$\pm2$ &8$\pm4$ &65\%$\pm$47\%\\
			& 0.00100 & $114\times10^{10}$ &$\mathbf{49}\pm56\times10^{10}$ & 43\% &4$\pm2$ &9$\pm4$ &65\%$\pm$47\%\\
			\midrule\multirow{9}{*}{5}& 1.00000 & $\mathbf{33}\times10^{5}$ &$36\times10^{5}$ & 109\% &7$\pm1$ &38$\pm4$ &0\%\\
			& 0.50000 & $\mathbf{13}\times10^{6}$ &$14\times10^{6}$ & 108\% &7$\pm1$ &38$\pm4$ &0\%\\
			& 0.25000 & $\mathbf{52}\times10^{6}$ &$57\times10^{6}$ & 110\% &7$\pm1$ &37$\pm4$ &0\%\\
			& 0.12500 & $\mathbf{21}\times10^{7}$ &23$\pm1\times10^{7}$ & 110\% &7$\pm1$ &37$\pm4$ &0.5\%$\pm$5\%\\
			& 0.06250 & $83\times10^{7}$ &$\mathbf{82}\pm24\times10^{7}$ & 99\% &7$\pm2$ &36$\pm6$ &12\%$\pm$33\%\\
			& 0.03125 & $33\times10^{8}$ &$\mathbf{23}\pm17\times10^{8}$ & 70\% &6$\pm3$ &30$\pm14$ &40\%$\pm$49\%\\
			& 0.01000 & $33\times10^{9}$ &$\mathbf{22}\pm16\times10^{9}$ & 67\% &6$\pm3$ &30$\pm11$ &40\%$\pm$49\%\\
			& 0.00500 & $130\times10^{9}$ &$\mathbf{88}\pm65\times10^{9}$ & 68\% &6$\pm3$ &30$\pm11$ &40\%$\pm$49\%\\
			& 0.00100 & $33\times10^{11}$ &$\mathbf{22}\pm16\times10^{11}$ & 67\% &6$\pm3$ &29$\pm12$ &40\%$\pm$49\%\\
			\midrule\multirow{9}{*}{6}& 1.00000 & $\mathbf{34}\times10^{5}$ &$37\times10^{5}$ & 109\% &7$\pm2$ &40$\pm5$ &0\%\\
			& 0.50000 & $\mathbf{14}\times10^{6}$ &$15\times10^{6}$ & 107\% &7$\pm1$ &40$\pm5$ &0\%\\
			& 0.25000 & $\mathbf{55}\times10^{6}$ &$60\times10^{6}$ & 109\% &7$\pm2$ &40$\pm5$ &0\%\\
			& 0.12500 & $22\times10^{7}$ &$\mathbf{21}\pm7\times10^{7}$ & 95\% &7$\pm2$ &38$\pm6$ &17\%$\pm$35\%\\
			& 0.06250 & $87\times10^{7}$ &$\mathbf{58}\pm40\times10^{7}$ & 67\% &6$\pm2$ &34$\pm11$ &47\%$\pm$49\%\\
			& 0.03125 & $35\times10^{8}$ &$\mathbf{12}\pm16\times10^{8}$ & 34\% &4$\pm2$ &26$\pm11$ &72\%$\pm$44\%\\
			& 0.01000 & $34\times10^{9}$ &$\mathbf{12}\pm15\times10^{9}$ & 35\% &4$\pm2$ &27$\pm12$ &71\%$\pm$44\%\\
			& 0.00500 & $136\times10^{9}$ &$\mathbf{49}\pm61\times10^{9}$ & 36\% &4$\pm2$ &26$\pm11$ &71\%$\pm$44\%\\
			& 0.00100 & $34\times10^{11}$ &$\mathbf{12}\pm15\times10^{11}$ & 35\% &4$\pm2$ &26$\pm12$ &71\%$\pm$44\%\\
			
		\end{tabular}
	\end{center}
	\caption{$s$-$t$ shortest paths results for graphs 1-6. Each experiment was repeated 100 times. Standard deviations are omitted when they are smaller than $1\%$ of the average. }\label{tab:paths_large1}
\end{table}

\begin{table}
	\begin{center}
		\resizebox{0.9\textwidth}{!}{
			\begin{tabular}{cc|cc|c|ccc}
				\toprule
				Graph & $\epsilon$ &   \multicolumn{2}{c|}{Sample size} & Sample size ratio & \multicolumn{3}{c}{\algname\ }\\
				ID & &  Naive & \algname\   & (\algname$/$ Naive) & Oracle calls & Time (millisec) & Accepted early\\ 
				\toprule
				
				\multirow{9}{*}{7}& 1.00000 & $\mathbf{21}\times10^{6}$ &23$\pm4\times10^{6}$ & 110\% &25$\pm8$ &79$\pm14$ &7\%$\pm$16\%\\
				& 0.50000 & $\mathbf{86}\times10^{6}$ &91$\pm17\times10^{6}$ & 106\% &22$\pm7$ &77$\pm14$ &7\%$\pm$16\%\\
				& 0.25000 & $\mathbf{34}\times10^{7}$ &36$\pm7\times10^{7}$ & 106\% &22$\pm7$ &76$\pm13$ &6\%$\pm$16\%\\
				& 0.12500 & $14\times10^{8}$ &$\mathbf{14}\pm3\times10^{8}$ & 100\% &22$\pm7$ &76$\pm13$ &7\%$\pm$16\%\\
				& 0.06250 & $\mathbf{55}\times10^{8}$ &57$\pm11\times10^{8}$ & 104\% &22$\pm7$ &77$\pm18$ &8\%$\pm$17\%\\
				& 0.03125 & $22\times10^{9}$ &$\mathbf{22}\pm5\times10^{9}$ & 100\% &21$\pm7$ &75$\pm13$ &11\%$\pm$20\%\\
				& 0.01000 & $21\times10^{10}$ &$\mathbf{19}\pm7\times10^{10}$ & 90\% &20$\pm7$ &74$\pm24$ &23\%$\pm$32\%\\
				& 0.00500 & $86\times10^{10}$ &$\mathbf{58}\pm28\times10^{10}$ & 67\% &19$\pm7$ &66$\pm16$ &63\%$\pm$39\%\\
				& 0.00100 & $215\times10^{11}$ &$\mathbf{88}\pm98\times10^{11}$ & 41\% &12$\pm8$ &48$\pm30$ &64\%$\pm$40\%\\
				\midrule\multirow{11}{*}{8}& 1.00000 & $\mathbf{13}\pm16\times10^{7}$ &14$\pm19\times10^{7}$ & 108\% &23$\pm13$ &80$\pm36$ &17\%$\pm$23\%\\
				& 0.50000 & $\mathbf{52}\pm65\times10^{7}$ &58$\pm74\times10^{7}$ & 112\% &20$\pm11$ &78$\pm33$ &18\%$\pm$23\%\\
				& 0.25000 & $\mathbf{21}\pm26\times10^{8}$ &23$\pm30\times10^{8}$ & 110\% &20$\pm11$ &77$\pm33$ &19\%$\pm$23\%\\
				& 0.12500 & $\mathbf{83}\pm104\times10^{8}$ &92$\pm119\times10^{8}$ & 111\% &18$\pm10$ &77$\pm33$ &19\%$\pm$24\%\\
				& 0.06250 & $\mathbf{33}\pm42\times10^{9}$ &37$\pm48\times10^{9}$ & 112\% &17$\pm9$ &75$\pm30$ &20\%$\pm$24\%\\
				& 0.03125 & $\mathbf{13}\pm17\times10^{10}$ &15$\pm19\times10^{10}$ & 115\% &16$\pm8$ &75$\pm32$ &22\%$\pm$27\%\\
				& 0.01000 & $\mathbf{13}\pm16\times10^{11}$ &14$\pm19\times10^{11}$ & 108\% &16$\pm8$ &74$\pm33$ &28\%$\pm$30\%\\
				& 0.00500 & $\mathbf{52}\pm65\times10^{11}$ &54$\pm73\times10^{11}$ & 104\% &15$\pm8$ &75$\pm38$ &33\%$\pm$32\%\\
				& 0.00100 & $13\pm16\times10^{13}$ &$\mathbf{10}\pm17\times10^{13}$ & 77\% &13$\pm7$ &79$\pm47$ &55\%$\pm$35\%\\
				& 1.00$\times10^{-4}$ & $129\pm163\times10^{14}$ &$\mathbf{52}\pm115\times10^{14}$ & 40\% &9$\pm6$ &69$\pm44$ &75\%$\pm$33\%\\
				& 0.10$\times10^{-4}$ & $129\pm163\times10^{16}$ &$\mathbf{19}\pm56\times10^{16}$ & 15\% &5$\pm4$ &46$\pm32$ &85\%$\pm$28\%\\
				\midrule\multirow{9}{*}{9}& 1.00000 & $12\times10^{8}$ &$\mathbf{12}\pm2\times10^{8}$ & 100\% &50$\pm15$ &2033$\pm521$ &9\%$\pm$15\%\\
				& 0.50000 & $46\times10^{8}$ &$\mathbf{46}\pm7\times10^{8}$ & 100\% &50$\pm15$ &2015$\pm535$ &15\%$\pm$19\%\\
				& 0.25000 & $19\times10^{9}$ &$\mathbf{17}\pm4\times10^{9}$ & 89\% &48$\pm15$ &1893$\pm486$ &28\%$\pm$28\%\\
				& 0.12500 & $74\times10^{9}$ &$\mathbf{52}\pm23\times10^{9}$ & 70\% &44$\pm16$ &1770$\pm494$ &53\%$\pm$35\%\\
				& 0.06250 & $30\times10^{10}$ &$\mathbf{14}\pm10\times10^{10}$ & 47\% &38$\pm14$ &1554$\pm473$ &71\%$\pm$33\%\\
				& 0.03125 & $119\times10^{10}$ &$\mathbf{31}\pm35\times10^{10}$ & 26\% &31$\pm15$ &1278$\pm487$ &84\%$\pm$30\%\\
				& 0.01000 & $116\times10^{11}$ &$\mathbf{22}\pm38\times10^{11}$ & 19\% &21$\pm15$ &904$\pm514$ &84\%$\pm$31\%\\
				& 0.00500 & $464\times10^{11}$ &$\mathbf{82}\pm153\times10^{11}$ & 18\% &17$\pm15$ &781$\pm520$ &84\%$\pm$29\%\\
				& 0.00100 & $116\times10^{13}$ &$\mathbf{21}\pm38\times10^{13}$ & 18\% &17$\pm14$ &830$\pm626$ &84\%$\pm$29\%\\
				\midrule\multirow{9}{*}{10}& 1.00000 & $\mathbf{14}\times10^{8}$ &15$\pm2\times10^{8}$ & 107\% &53$\pm14$ &2150$\pm553$ &8\%$\pm$11\%\\
				& 0.50000 & $57\times10^{8}$ &$\mathbf{56}\pm10\times10^{8}$ & 98\% &52$\pm14$ &2145$\pm556$ &13\%$\pm$19\%\\
				& 0.25000 & $23\times10^{9}$ &$\mathbf{20}\pm5\times10^{9}$ & 87\% &51$\pm14$ &2084$\pm528$ &27\%$\pm$30\%\\
				& 0.12500 & $91\times10^{9}$ &$\mathbf{67}\pm26\times10^{9}$ & 74\% &48$\pm15$ &1941$\pm512$ &51\%$\pm$38\%\\
				& 0.06250 & $36\times10^{10}$ &$\mathbf{18}\pm13\times10^{10}$ & 50\% &42$\pm14$ &1719$\pm493$ &72\%$\pm$36\%\\
				& 0.03125 & $145\times10^{10}$ &$\mathbf{44}\pm46\times10^{10}$ & 30\% &35$\pm14$ &1436$\pm493$ &84\%$\pm$29\%\\
				& 0.01000 & $142\times10^{11}$ &$\mathbf{28}\pm47\times10^{11}$ & 20\% &22$\pm15$ &1009$\pm514$ &85\%$\pm$29\%\\
				& 0.00500 & $57\times10^{12}$ &$\mathbf{10}\pm19\times10^{12}$ & 18\% &18$\pm15$ &844$\pm501$ &85\%$\pm$29\%\\
				& 0.00100 & $142\times10^{13}$ &$\mathbf{26}\pm48\times10^{13}$ & 18\% &17$\pm15$ &852$\pm587$ &85\%$\pm$29\%\\
				\midrule\multirow{9}{*}{11}& 1.00000 & $\mathbf{73}\times10^{7}$ &75$\pm9\times10^{7}$ & 103\% &48$\pm14$ &1299$\pm336$ &10\%$\pm$12\%\\
				& 0.50000 & $29\times10^{8}$ &$\mathbf{28}\pm4\times10^{8}$ & 97\% &48$\pm14$ &1272$\pm334$ &17\%$\pm$19\%\\
				& 0.25000 & $12\times10^{9}$ &$\mathbf{10}\pm2\times10^{9}$ & 83\% &46$\pm14$ &1219$\pm332$ &31\%$\pm$26\%\\
				& 0.12500 & $47\times10^{9}$ &$\mathbf{33}\pm13\times10^{9}$ & 70\% &42$\pm14$ &1130$\pm319$ &51\%$\pm$33\%\\
				& 0.06250 & $188\times10^{9}$ &$\mathbf{94}\pm60\times10^{9}$ & 50\% &37$\pm13$ &1005$\pm304$ &66\%$\pm$32\%\\
				& 0.03125 & $75\times10^{10}$ &$\mathbf{19}\pm21\times10^{10}$ & 25\% &30$\pm13$ &832$\pm296$ &85\%$\pm$28\%\\
				& 0.01000 & $73\times10^{11}$ &$\mathbf{13}\pm22\times10^{11}$ & 18\% &20$\pm13$ &592$\pm295$ &85\%$\pm$27\%\\
				& 0.00500 & $294\times10^{11}$ &$\mathbf{49}\pm91\times10^{11}$ & 17\% &16$\pm12$ &494$\pm285$ &85\%$\pm$27\%\\
				& 0.00100 & $73\times10^{13}$ &$\mathbf{12}\pm23\times10^{13}$ & 16\% &16$\pm12$ &518$\pm333$ &85\%$\pm$28\%\\
				\midrule\multirow{9}{*}{12}& 1.00000 & $12\times10^{8}$ &$\mathbf{12}\pm1\times10^{8}$ & 100\% &50$\pm15$ &1759$\pm477$ &10\%$\pm$17\%\\
				& 0.50000 & $46\times10^{8}$ &$\mathbf{44}\pm9\times10^{8}$ & 96\% &50$\pm15$ &1722$\pm473$ &20\%$\pm$25\%\\
				& 0.25000 & $18\times10^{9}$ &$\mathbf{15}\pm5\times10^{9}$ & 83\% &48$\pm16$ &1655$\pm488$ &47\%$\pm$37\%\\
				& 0.12500 & $74\times10^{9}$ &$\mathbf{43}\pm23\times10^{9}$ & 58\% &43$\pm16$ &1488$\pm457$ &65\%$\pm$35\%\\
				& 0.06250 & $30\times10^{10}$ &$\mathbf{11}\pm10\times10^{10}$ & 37\% &36$\pm15$ &1275$\pm447$ &85\%$\pm$28\%\\
				& 0.03125 & $118\times10^{10}$ &$\mathbf{19}\pm29\times10^{10}$ & 16\% &28$\pm15$ &1011$\pm434$ &93\%$\pm$22\%\\
				& 0.01000 & $115\times10^{11}$ &$\mathbf{10}\pm28\times10^{11}$ & 9\% &17$\pm13$ &692$\pm385$ &93\%$\pm$21\%\\
				& 0.00500 & $461\times10^{11}$ &$\mathbf{40}\pm115\times10^{11}$ & 9\% &13$\pm11$ &577$\pm360$ &93\%$\pm$21\%\\
				& 0.00100 & $1153\times10^{12}$ &$\mathbf{90}\pm278\times10^{12}$ & 8\% &13$\pm11$ &561$\pm373$ &93\%$\pm$22\%\\
			\end{tabular}
		}
	\end{center}
	\caption{$s$-$t$ shortest paths results for graphs 7-12. Each experiment was repeated 100 times. Standard deviations are omitted when they are smaller than $1\%$ of the average.}\label{tab:paths_large2}
\end{table}

\begin{table}[h]
	\begin{center}
		\resizebox{0.9\textwidth}{!}{
			\begin{tabular}{cc|cc|c|ccc}
				\toprule
				Graph & $\epsilon$ &   \multicolumn{2}{c|}{Sample size} & Sample size ratio & \multicolumn{3}{c}{\algname\ }\\
				ID & &  Naive & \algname\   & (\algname$/$ Naive) & Oracle calls & Time (millisec) & Accepted early\\ 
				\toprule

				\multirow{8}{*}{1} & 0.50000 & $\mathbf{34}\times10^{5}$ &$37\times10^{5}$ & 109\% &37 &250$\pm8$ &0\%\\
				& 0.25000 & $\mathbf{13}\times10^{6}$ &$15\times10^{6}$ & 115\% &36 &246$\pm7$ &0\%\\
				& 0.12500 & $54\times10^{6}$ &$\mathbf{47}\times10^{6}$ & 87\% &36$\pm1$ &245$\pm6$ &11\%\\
				& 0.06250 & $21\times10^{7}$ &$\mathbf{16}\times10^{7}$ & 76\% &27$\pm1$ &185$\pm4$ &11\%\\
				& 0.03125 & $86\times10^{7}$ &$\mathbf{26}\pm4\times10^{7}$ & 30\% &26 &162$\pm5$ &32\%$\pm$3\%\\
				& 0.01000 & $84\times10^{8}$ &$\mathbf{18}\times10^{8}$ & 21\% &22 &103$\pm3$ &33\%\\
				& 0.00500 & $336\times10^{8}$ &$\mathbf{73}\pm1\times10^{8}$ & 22\% &22 &106$\pm4$ &33\%\\
				& 0.00100 & $84\times10^{10}$ &$\mathbf{18}\times10^{10}$ & 21\% &22 &102$\pm7$ &33\%\\
				\midrule\multirow{8}{*}{2} & 0.50000 & $\mathbf{19}\times10^{6}$ &$21\times10^{6}$ & 111\% &37$\pm1$ &635$\pm19$ &0\%\\
				& 0.25000 & $\mathbf{75}\times10^{6}$ &$83\times10^{6}$ & 111\% &37 &641$\pm11$ &0\%\\
				& 0.12500 & $\mathbf{30}\times10^{7}$ &$33\times10^{7}$ & 110\% &37 &648$\pm8$ &0\%\\
				& 0.06250 & $\mathbf{12}\times10^{8}$ &$13\times10^{8}$ & 108\% &37 &644$\pm5$ &0\%\\
				& 0.03125 & $48\times10^{8}$ &$\mathbf{34}\times10^{8}$ & 71\% &37 &643$\pm7$ &7\%\\
				& 0.01000 & $47\times10^{9}$ &$\mathbf{12}\times10^{9}$ & 26\% &36 &501$\pm6$ &27\%\\
				& 0.00500 & $187\times10^{9}$ &$\mathbf{24}\times10^{9}$ & 13\% &33 &407$\pm4$ &47\%\\
				& 0.00100 & $468\times10^{10}$ &$\mathbf{15}\times10^{10}$ & 3\% &18 &204$\pm3$ &53\%\\
				\midrule\multirow{8}{*}{3} & 0.50000 & $\mathbf{97}\times10^{6}$ &$110\times10^{6}$ & 113\% &63$\pm3$ &1684$\pm89$ &0\%\\
				& 0.25000 & $\mathbf{39}\times10^{7}$ &$44\times10^{7}$ & 113\% &64$\pm3$ &1716$\pm92$ &0\%\\
				& 0.12500 & $\mathbf{16}\times10^{8}$ &$18\times10^{8}$ & 112\% &62$\pm2$ &1703$\pm55$ &0\%\\
				& 0.06250 & $62\times10^{8}$ &$\mathbf{36}\times10^{8}$ & 58\% &63$\pm2$ &1688$\pm72$ &26\%\\
				& 0.03125 & $249\times10^{8}$ &$\mathbf{90}\times10^{8}$ & 36\% &59$\pm3$ &1498$\pm62$ &30\%\\
				& 0.01000 & $244\times10^{9}$ &$\mathbf{67}\times10^{9}$ & 27\% &40$\pm2$ &866$\pm40$ &30\%\\
				& 0.00500 & $97\times10^{10}$ &$\mathbf{26}\times10^{10}$ & 27\% &34$\pm1$ &637$\pm33$ &30\%\\
				& 0.00100 & $244\times10^{11}$ &$\mathbf{66}\times10^{11}$ & 27\% &34$\pm2$ &622$\pm43$ &30\%\\
				\midrule\multirow{8}{*}{4} & 0.50000 & $\mathbf{89}\times10^{6}$ &$100\times10^{6}$ & 112\% &72$\pm1$ &482$\pm9$ &0\%\\
				& 0.25000 & $36\times10^{7}$ &$\mathbf{36}\times10^{7}$ & 100\% &72$\pm1$ &485$\pm7$ &4\%\\
				& 0.12500 & $\mathbf{14}\times10^{8}$ &$15\times10^{8}$ & 107\% &71$\pm1$ &574$\pm12$ &4\%\\
				& 0.06250 & $57\times10^{8}$ &$\mathbf{46}\times10^{8}$ & 81\% &69$\pm2$ &671$\pm17$ &9\%\\
				& 0.03125 & $23\times10^{9}$ &$\mathbf{18}\times10^{9}$ & 78\% &66$\pm2$ &687$\pm32$ &9\%\\
				& 0.01000 & $\mathbf{22}\times10^{10}$ &$23\times10^{10}$ & 105\% &64$\pm3$ &676$\pm33$ &4\%\\
				& 0.00500 & $89\times10^{10}$ &$\mathbf{68}\times10^{10}$ & 76\% &65$\pm2$ &695$\pm22$ &9\%\\
				& 0.00100 & $22\times10^{12}$ &$\mathbf{17}\times10^{12}$ & 77\% &66$\pm2$ &756$\pm60$ &9\%\\
				\midrule\multirow{8}{*}{5} & 0.50000 & $\mathbf{71}\times10^{7}$ &$80\times10^{7}$ & 113\% &140$\pm4$ &10$\pm0\times10^{3}$ &0\%\\
				& 0.25000 & $\mathbf{28}\times10^{8}$ &$32\times10^{8}$ & 114\% &141$\pm3$ &10$\pm0\times10^{3}$ &0\%\\
				& 0.12500 & $\mathbf{11}\times10^{9}$ &$12\times10^{9}$ & 109\% &140$\pm3$ &10$\pm0\times10^{3}$ &3\%\\
				& 0.06250 & $45\times10^{9}$ &$\mathbf{35}\times10^{9}$ & 78\% &137$\pm4$ &10$\pm0\times10^{3}$ &29\%\\
				& 0.03125 & $181\times10^{9}$ &$\mathbf{72}\times10^{9}$ & 40\% &115$\pm3$ &8448$\pm199$ &29\%\\
				& 0.01000 & $177\times10^{10}$ &$\mathbf{57}\times10^{10}$ & 32\% &91$\pm2$ &6236$\pm174$ &29\%\\
				& 0.00500 & $71\times10^{11}$ &$\mathbf{23}\times10^{11}$ & 32\% &81$\pm3$ &4891$\pm178$ &29\%\\
				& 0.00100 & $177\times10^{12}$ &$\mathbf{57}\times10^{12}$ & 32\% &69$\pm3$ &3541$\pm126$ &29\%\\
				\midrule\multirow{8}{*}{6} & 0.50000 & $\mathbf{71}\times10^{7}$ &$80\times10^{7}$ & 113\% &104$\pm3$ &3114$\pm219$ &0\%\\
				& 0.25000 & $\mathbf{28}\times10^{8}$ &$32\times10^{8}$ & 114\% &104$\pm3$ &2984$\pm117$ &0\%\\
				& 0.12500 & $\mathbf{11}\times10^{9}$ &$13\times10^{9}$ & 118\% &102$\pm5$ &3014$\pm203$ &0\%\\
				& 0.06250 & $45\times10^{9}$ &$\mathbf{45}\times10^{9}$ & 100\% &103$\pm4$ &3026$\pm168$ &5\%\\
				& 0.03125 & $18\times10^{10}$ &$\mathbf{17}\times10^{10}$ & 94\% &100$\pm5$ &3114$\pm173$ &5\%\\
				& 0.01000 & $18\times10^{11}$ &$\mathbf{16}\times10^{11}$ & 89\% &92$\pm5$ &3041$\pm146$ &5\%\\
				& 0.00500 & $71\times10^{11}$ &$\mathbf{66}\times10^{11}$ & 93\% &88$\pm4$ &3011$\pm120$ &5\%\\
				& 0.00100 & $18\times10^{13}$ &$\mathbf{16}\times10^{13}$ & 89\% &85$\pm4$ &2674$\pm287$ &5\%\\
			\end{tabular}
		}
		
	\end{center}
	\caption{Maximum weight matching results for graphs 1-6. Each experiment was repeated 10 times. Standard deviations are omitted when they are smaller than $1\%$ of the average.}
	\label{tab:matching_large1}
\end{table} 

\begin{table}[h]
	\begin{center}
		\begin{tabular}{cc|cc|c|ccc}
			\toprule
			Graph & $\epsilon$ &   \multicolumn{2}{c|}{Sample size} & Sample size ratio & \multicolumn{3}{c}{\algname\ }\\
			ID & &  Naive & \algname\   & (\algname$/$ Naive) & Oracle calls & Time (millisec) & Accepted early\\ 
			\toprule

			\multirow{8}{*}{7} & 0.50000 & $\mathbf{13}\times10^{9}$ &$15\times10^{9}$ & 115\% &492$\pm5$ &51$\pm1\times10^{4}$ &0\%\\
			& 0.25000 & $\mathbf{51}\times10^{9}$ &$60\times10^{9}$ & 118\% &498$\pm5$ &55$\pm4\times10^{4}$ &0\%\\
			& 0.12500 & $\mathbf{20}\times10^{10}$ &$23\times10^{10}$ & 115\% &497$\pm9$ &53$\pm1\times10^{4}$ &1\%\\
			& 0.06250 & $81\times10^{10}$ &$\mathbf{81}\times10^{10}$ & 100\% &498$\pm9$ &52$\pm1\times10^{4}$ &2\%\\
			& 0.03125 & $32\times10^{11}$ &$\mathbf{29}\times10^{11}$ & 91\% &497$\pm6$ &52$\pm1\times10^{4}$ &5\%\\
			& 0.01000 & $32\times10^{12}$ &$\mathbf{16}\times10^{12}$ & 50\% &479$\pm6$ &48$\pm1\times10^{4}$ &22\%\\
			& 0.00500 & $127\times10^{12}$ &$\mathbf{34}\times10^{12}$ & 27\% &453$\pm6$ &43$\pm1\times10^{4}$ &38\%\\
			& 0.00100 & $317\times10^{13}$ &$\mathbf{59}\pm1\times10^{13}$ & 19\% &356$\pm7$ &28$\pm1\times10^{4}$ &38\%$\pm$1\%\\
			\midrule\multirow{8}{*}{8} & 0.50000 & $\mathbf{95}\times10^{8}$ &$113\times10^{8}$ & 119\% &414$\pm6$ &33$\pm1\times10^{4}$ &0\%\\
			& 0.25000 & $\mathbf{38}\times10^{9}$ &$44\times10^{9}$ & 116\% &413$\pm3$ &34$\pm1\times10^{4}$ &1\%\\
			& 0.12500 & $\mathbf{15}\times10^{10}$ &$18\times10^{10}$ & 120\% &414$\pm2$ &33$\pm0\times10^{4}$ &1\%\\
			& 0.06250 & $\mathbf{61}\times10^{10}$ &$70\times10^{10}$ & 115\% &413$\pm3$ &34$\pm1\times10^{4}$ &1\%\\
			& 0.03125 & $\mathbf{24}\times10^{11}$ &$25\times10^{11}$ & 104\% &414$\pm2$ &33$\pm0\times10^{4}$ &3\%\\
			& 0.01000 & $24\times10^{12}$ &$\mathbf{16}\times10^{12}$ & 67\% &406$\pm2$ &31$\pm0\times10^{4}$ &14\%\\
			& 0.00500 & $95\times10^{12}$ &$\mathbf{38}\times10^{12}$ & 40\% &391$\pm1$ &30$\pm1\times10^{4}$ &27\%\\
			& 0.00100 & $238\times10^{13}$ &$\mathbf{34}\times10^{13}$ & 14\% &307$\pm1$ &20$\pm1\times10^{4}$ &47\%\\

		\end{tabular}
		
	\end{center}
	\caption{Maximum weight matching results for graphs 7-8. Each experiment was repeated 10 times. Standard deviations are omitted when they are smaller than $1\%$ of the average.}
	\label{tab:matching_large2}
\end{table}

\end{document}